\documentclass[11pt]{article}
\usepackage{authblk}

\title{Spectral partitioning for $k$-block averaging kernels of finite Markov chains}
\author[1]{Michael C.H. Choi\thanks{Email: mchchoi@nus.edu.sg, corresponding author}}
\author[1]{Youjia Wang\thanks{Email: e1124868@u.nus.edu}}
\affil[1]{Department of Statistics and Data Science, National University of Singapore, Level 7, 6 Science Drive 2, 117546, Singapore}

\usepackage[margin=1in]{geometry}
\usepackage{amsmath,amssymb,amsthm}
\usepackage{float}
\floatstyle{ruled}
\newfloat{algorithm}{tbp}{loa}
\floatname{algorithm}{Algorithm}
\usepackage[authoryear,round]{natbib}
\usepackage{xurl}
\usepackage{xcolor}
\usepackage[
colorlinks=true,
citecolor=blue,
linkcolor=red,
urlcolor=blue
]{hyperref}
\usepackage{array,booktabs,tabularx}
\usepackage{graphicx}
\graphicspath{{figures_v10/}}

\newtheorem{theorem}{Theorem}[section]
\newtheorem{proposition}[theorem]{Proposition}
\newtheorem{lemma}[theorem]{Lemma}
\newtheorem{corollary}[theorem]{Corollary}

\newcommand{\1}{\mathbf{1}}
\newcommand{\X}{\mathcal{X}}
\providecommand{\llbracket}{\mathopen{[\mkern-2.7mu[}}
\providecommand{\rrbracket}{\mathclose{]\mkern-2.7mu]}}

\begin{document}

\maketitle

\begin{abstract}
We develop spectral algorithms for selecting state-space partitions
that define averaging kernels for finite, ergodic and reversible Markov
chains.  For a partition $\mathcal O$, the Gibbs kernel
$G_{\mathcal O}$ resamples within the current block from the
stationary conditional distribution; when this update is tractable,
composing or mixing it with a baseline kernel $P$ can accelerate
convergence.  We select $\mathcal O$ by rounding the bottom
nonconstant eigenfunctions of $P^2$, or the algebraically smallest
eigenfunctions of $P$ for additive mixtures, using weighted
$k$-means.  For
$F(\mathcal O)=\|G_{\mathcal O}P-\Pi\|_{F,\pi}^2$, we derive exact
trace and normalized-cut representations and show that $F$ equals the
Pearson $\chi^2$-mutual information between the initial block label
and the state after one transition, giving this matrix objective a
natural probabilistic interpretation.  In the two-block case, a
threshold sweep exactly solves the associated one-dimensional
weighted two-means rounding problem.  For general $k \geq 2$, weighted
$k$-means rounds the bottom $(k-1)$-dimensional embedding, after which
candidates are rescored by $F$; the rounding distortion is a distance
between subspaces that yields spectral approximation bounds.  We extend the framework to additive mixtures, finite-horizon objectives, and discounted infinite-horizon objectives.  In contrast to classical normalized spectral clustering, which uses top nonconstant modes to find
low-flow persistent clusters, our method uses bottom modes to favor
large normalized cross-block flow and rapid loss of block-label
information.  Experiments on a controlled-spectrum graph, a
mean-field Ising model, and Bayesian variable selection show
notable per-iteration improvements in convergence and statistical
estimation.

\end{abstract}

\tableofcontents

\section{Introduction}
\label{sec:introduction}

A common way to improve a Markov chain is to supplement its local
moves with an update that averages over a larger part of the state
space.  Let $P$ be a finite, ergodic, reversible Markov kernel with
stationary distribution $\pi$, and let
$\mathcal O=\{O_1,\ldots,O_k\}$ be a partition.  The associated Gibbs
kernel $G_{\mathcal O}$ draws from $\pi$ conditional on the block that
contains the current state.  Whenever these conditional distributions
can be sampled efficiently, one can implement an averaged
chain such as $G_{\mathcal O}P$, or the additive mixture
$A_{\mathcal O}=(P+G_{\mathcal O})/2$.  A suitable partition can then
suppress troublesome modes and speed up convergence.  The practical
question is how to choose that partition without searching over all
set partitions of the state space.

This question is motivated by recent work on group-averaged Markov
chains.  \citet{ChoiLimWang2025GroupAveragedII} study orbit-based Gibbs,
Metropolis--Hastings, and Barker averaging kernels and develop theory
and heuristics for tuning the group action.  \citet{LimChoi2026TwoBlock}
formulate optimal two-block selection under Kullback--Leibler and
Frobenius criteria and give combinatorial approximation schemes, while
\citet{LimChoi2026Additive} study additive averaging kernels and the
joint roles of the partition and mixture weight.  These results show
the potential of averaging but also make partition selection a central
computational bottleneck.  The present paper addresses this bottleneck
with spectral rounding algorithms.

Our starting objective is the squared $\pi$-weighted Frobenius one-step distance to stationarity given by
\[
    F(\mathcal O)=\|G_{\mathcal O}P-\Pi\|_{F,\pi}^2,
\]
where $\Pi(x,y) = \pi(y)$ is the stationary transition matrix with each row being $\pi$.  The projection structure of
$G_{\mathcal O}$ gives
$F(\mathcal O)=\operatorname{Tr}((G_{\mathcal O}-\Pi)P^2)$, so its
continuous relaxation is governed by the bottom nonconstant
eigenfunctions of $P^2$.  We round this eigenspace to a block-constant
subspace, thereby obtaining a partition and hence an averaging kernel.
For additive mixtures, the corresponding relaxation instead uses the
algebraically smallest eigenfunctions of $P$.  The punch line is that
eigenfunctions of the original chain are used not only to diagnose
convergence, but to construct a new, simulable kernel intended to mix
faster.


Eigenfunctions of Markov processes have also been studied directly
for the geometric and probabilistic information that they encode.
For finite birth--death processes and more generally, Markov
processes on trees, \citet{Miclo2008EigenfunctionsTrees} describes the
nodal and monotonicity structure of nonconstant eigenfunctions.
In a discrete absorbing setting,
\citet{DiaconisMiclo2016DirichletAmplitude} bound the amplitude of the
positive first Dirichlet eigenvector and relate these estimates to
convergence toward quasi-stationarity.  Our use of eigenfunctions is
complementary: we treat bottom eigenfunctions as optimization
coordinates, round them into a state-space partition, and use that
partition to construct an averaging kernel.

The use of Frobenius distance as a Markov-chain design criterion also
appears in recent work on nearest reversible chains.
\citet{DurastanteGnazzoMeini2025NearestReversible} seek, among kernels
with a prescribed stationary distribution, a reversible transition
matrix nearest to a given kernel in Frobenius norm, and solve the
resulting problem by Riemannian optimization.
\citet{CipollaDurastanteGnazzoMeini2026NearestReversible} incorporate
sparsity constraints and formulate the corresponding matrix-nearness
problem as a quadratic program.  Their decision variable is the
transition matrix and their aim is to enforce reversibility, whereas
our baseline $P$ is already reversible: we optimize over partitions
that induce averaging projections, with the aim of accelerating
mixing. 

State-space partitioning by spectral clustering has also been used to
parallelize MCMC.  \citet{BasseSmithPillai2016} use conventional
spectral clusters to define restricted chains on different regions and
combine their weighted estimates, obtaining improvements for
multimodal targets and over naive parallelization.  Their construction
seeks slowly communicating regions that can be explored separately,
whereas our partition defines an averaging projection within the
transition kernel and, as explained next, favors large rather than
small cross-block flow.

This spectral orientation is opposite to classical normalized
spectral clustering \citep{ShiMalik2000,NgJordanWeiss2001}.  Classical
methods use so-called low-frequency eigenfunctions of $I-P^2$, equivalently the
largest nonconstant eigenfunctions of $P^2$, to identify clusters with
small boundary.  Our Frobenius objective maximizes a multiway
normalized cut of $P^2$ and therefore uses its smallest eigenfunctions,
which are high-frequency directions of $I-P^2$.  The algorithms may
look similar (thresholding in one dimension and geometric rounding in
higher dimensions), but their objectives and guarantees are reversed.

The main contributions are as follows.

\begin{itemize}
    \item For two blocks, we show that thresholding a bottom
    eigenfunction of $P^2$ solves the associated one-dimensional
    weighted two-means problem.  We then bound the suboptimality of the
    spectral sweep using its rounding error, eigengaps, and spectral
    width.

    \item More generally for $k \geq 2$ blocks, we embed the states with the bottom $k-1$
    eigenfunctions of $P^2$ and round by $\pi$-weighted $k$-means.  Its
    distortion is exactly a projection distance, equivalently a sum
    of squared principal-angle sines, and spectral-sandwich bounds
    transfer this geometric error to the original objective.

    \item We extend the framework to additive mixtures using the
    bottom eigenfunctions of $P$ and to finite- and discounted
    infinite-horizon criteria defined by spectral filters of $P^2$. Numerical experiments on a controlled-spectrum dumbbell graph and heat-bath chains for a mean-field Ising model show that these
    criteria can select different partitions and substantially improve
    convergence when the averaging updates are available.  An exact
    Bayesian variable-selection benchmark further connects this
    improvement to statistical estimation and demonstrates the need
    for balanced partitions.
\end{itemize}

The remainder of the paper is organized as follows.
Section~\ref{sec:preliminaries} fixes the notation and develops the
spectral relaxation.  Sections~\ref{sec:two-block} and
\ref{sec:k-block} treat the two-block and general $k$-block problems,
respectively.  Section~\ref{sec:additive-mixture} develops the
additive-mixture analogue, and Section~\ref{sec:multi-horizon}
introduces multi-horizon objectives.  Section~\ref{sec:numerical-experiments}
presents experiments on a controlled-spectrum graph, a mean-field
Ising model, and Bayesian variable selection with block-mass
constraints.

\section{Preliminaries}
\label{sec:preliminaries}

For $a,b\in\mathbb Z$ with $a\leq b$, we define
$\llbracket a,b\rrbracket:=\{a,a+1,\ldots,b\}$, and for
$N\in\mathbb N$ we denote
$\llbracket N\rrbracket:=\llbracket 1,N\rrbracket$.

\subsection{Reversible Markov kernels}

Let $\X$ be a finite state space with $n=|\X| \geq 3$, and let $\pi$ be a
strictly positive probability distribution on $\X$.  For
$A\subseteq\X$, let $\1_A$ denote its indicator function, and write
$\1:=\1_\X$ for the constant-one function.  Let
$\ell^2(\pi)$ be the space of real-valued functions on $\X$, equipped
with the inner product and norm
\[
    \langle f,g\rangle_\pi
    =\sum_{x\in\X}\pi(x)f(x)g(x),
    \qquad
    \|f\|_\pi^2=\langle f,f\rangle_\pi.
\]
We identify a Markov kernel $P$ with the operator
\[
    (Pf)(x)=\sum_{y\in\X}P(x,y)f(y).
\]
Throughout, $P$ is assumed to be ergodic (that is, irreducible and
aperiodic), stationary with respect to $\pi$, and reversible.  Thus,
for all $x,y\in\X$, the detailed balance conditions are satisfied with
\[
    \pi(x)P(x,y)=\pi(y)P(y,x).
\]
As a result $P$ is self-adjoint on $\ell^2(\pi)$, its eigenvalues are real and
belong to $[-1,1]$, and $P\1=\1$.

The stationary projection is
\[
    \Pi f=\pi(f)\1,
    \qquad
    \pi(f)=\sum_{x\in\X}\pi(x)f(x).
\]
The centered subspace is
\begin{equation}
    \ell^2_0(\pi)
    :=\{f\in\ell^2(\pi):\pi(f)=0\}
    =\{f\in\ell^2(\pi):\langle f,\1\rangle_\pi=0\}.
    \label{eq:centered-subspace}
\end{equation}
Equivalently, $\ell^2_0(\pi)=\ker(\Pi)$.
As a Markov kernel, $\Pi(x,y)=\pi(y)$.  It satisfies
\[
    \Pi^2=\Pi,
    \qquad
    P\Pi=\Pi P=\Pi.
\]
For later use, if $u,v\in\ell^2(\pi)$, define the rank-one operator
\[
    (u\otimes_\pi v)f:=u\langle v,f\rangle_\pi.
\]

\subsection{Gibbs kernels associated with partitions}

Let
$
    \mathcal{O}=\{O_1,\ldots,O_k\}
$
be a partition of $\X$ into $k \in \llbracket n \rrbracket$ nonempty blocks, that is, $\mathcal{X} = \sqcup_{i=1}^k O_i$.  The associated Gibbs kernel induced by $\mathcal{O}$ is defined to be, for $f \in \ell^2(\pi)$,
\begin{equation*}
    (G_{\mathcal O}f)(x)
    :=\sum_{i=1}^k \1_{O_i}(x)
      \frac{1}{\pi(O_i)}
      \sum_{y\in O_i}\pi(y)f(y).
\end{equation*}
Equivalently, if $x,y\in\X$, then
\[
    G_{\mathcal O}(x,y)
    =\sum_{i=1}^k
      \1_{O_i}(x)\1_{O_i}(y)\frac{\pi(y)}{\pi(O_i)}.
\]
The operator $G_{\mathcal O}$ replaces a function by its $\pi$-mean
on each block.  It is therefore the orthogonal projection onto the
$k$-dimensional space of functions that are constant on every block.
Kernels of this form arise as orbit Gibbs kernels in group-averaged
Markov chains; their structural properties and the tuning of their
partitions are studied by
\citet{ChoiLimWang2025GroupAveragedII,LimChoi2026TwoBlock,LimChoi2026Additive}.
In particular,
\[
    G_{\mathcal O}^*=G_{\mathcal O},
    \qquad
    G_{\mathcal O}^2=G_{\mathcal O},
    \qquad
    G_{\mathcal O}\Pi=\Pi G_{\mathcal O}=\Pi.
\]
For a two-block partition $\{S,S^c\}$, we write $G_S$ in place of
$G_{\{S,S^c\}}$.

We use throughout the centered block projection and its range
\begin{equation}
    Q_{\mathcal O}:=G_{\mathcal O}-\Pi,
    \qquad
    \mathcal V_{\mathcal O}:=\operatorname{Ran}(Q_{\mathcal O}).
    \label{eq:QO-VO}
\end{equation}
Here $\operatorname{Ran}$ denotes the range, or image, of an operator.
The operator $Q_{\mathcal O}$ is the orthogonal projection onto the
$(k-1)$-dimensional subspace, that is, 
\[
    \mathcal V_{\mathcal O}
    =
    \{f\in\ell^2_0(\pi):
      f\text{ is constant on every }O_i\}.
\]

\subsection{The weighted Frobenius objective and its information-theoretic interpretation}

For an operator $A$ on $\ell^2(\pi)$, define its weighted Frobenius, or
Hilbert--Schmidt, norm by
\[
    \|A\|_{F,\pi}^2:=\operatorname{Tr}(A^*A).
\]
When $A$ is a matrix, this can equivalently be expressed in terms of the entries of $A$ as
\[
    \|A\|_{F,\pi}^2
    =\sum_{x,y\in\X}
      \frac{\pi(x)}{\pi(y)}A(x,y)^2.
\]

For any operator $B$ on $\ell^2(\pi)$ and sets
$A,D\subseteq\X$, write
\begin{equation}
    \mathcal Q_B(A,D)
    :=\sum_{x\in A}\sum_{y\in D}\pi(x)B(x,y).
    \label{eq:general-flow}
\end{equation}
When $B$ is a Markov kernel with stationary distribution $\pi$, this
is its stationary flow or the edge measure from $A$ to $D$.  For an entrywise nonnegative matrix
$B$ and a partition
$\mathcal O=\{O_1,\ldots,O_k\}$, write
\begin{equation*}
    \operatorname{Ncut}_B(\mathcal O)
    :=\sum_{i=1}^k
      \frac{\mathcal Q_B(O_i,O_i^c)}{\pi(O_i)}.
\end{equation*}
If $B\1=c\1$, the normalized block indicators give
\[
    \operatorname{Tr}(G_{\mathcal O}B)
    =\sum_{i=1}^k
      \frac{\mathcal Q_B(O_i,O_i)}{\pi(O_i)}
    =ck-\operatorname{Ncut}_B(\mathcal O).
\]
Since $\operatorname{Tr}(\Pi B)=c$, it follows that
\begin{equation}
    \operatorname{Tr}(Q_{\mathcal O}B)
    =c(k-1)-\operatorname{Ncut}_B(\mathcal O).
    \label{eq:general-trace-cut}
\end{equation}
In the sequel, we will take $B$ to be either $P$ or $P^2$ or the
filtered operators $B_w^T$ or $B_\gamma^\infty$ to be introduced in Section \ref{sec:multi-horizon} later.

The objective studied in subsequent sections is
\begin{equation}
    F(\mathcal O)
    :=\|G_{\mathcal O}P-\Pi\|_{F,\pi}^2.
    \label{eq:F-definition}
\end{equation}
The same value is obtained from $PG_{\mathcal O}$ because the two
operators are adjoints of each other and the assumption that $P$ is $\pi$-reversible.

Since $G_{\mathcal O}-\Pi$ is an orthogonal projection and
$G_{\mathcal O}P-\Pi=(G_{\mathcal O}-\Pi)P$, cyclicity of the trace
gives the useful identity
\begin{equation}
    F(\mathcal O)
    =\operatorname{Tr}\bigl((G_{\mathcal O}-\Pi)P^2\bigr).
    \label{eq:F-trace}
\end{equation}

Since $P^2\1=\1$, the general trace--cut identity
\eqref{eq:general-trace-cut} gives
\begin{equation}
    F(\mathcal O)
    =k-1-\operatorname{Ncut}_{P^2}(\mathcal O).
    \label{eq:F-cut}
\end{equation}
Thus minimizing $F$ is the same as maximizing the multiway
normalized-cut objective of the two-step chain $P^2$.

Although \eqref{eq:F-definition} introduces $F$ through a squared
weighted Frobenius norm, the same objective has an exact probabilistic
and information-theoretic interpretation.  For probability
distributions $\mu$ and $\nu$ on the same finite set, with $\nu$
strictly positive, define the Pearson chi-square divergence by
\[
    \chi^2(\mu\|\nu)
    :=\sum_y\frac{(\mu(y)-\nu(y))^2}{\nu(y)}.
\]
For finite-valued random variables $U$ and $V$, define their Pearson
chi-square mutual information by
\[
    I_{\chi^2}(U;V)
    :=\chi^2\!\left(
        \mathcal L(U,V)
        \,\middle\|\,
        \mathcal L(U)\otimes\mathcal L(V)
      \right),
\]
where $\mathcal L$ denotes probability law and $\otimes$ denotes the
product of probability measures.  Explicitly,
\[
\bigl(\mathcal L(U)\otimes\mathcal L(V)\bigr)(u,v)
=
\mathbb P(U=u)\mathbb P(V=v).
\]

Let $(X_t)_{t \in \mathbb{N}\cup\{0\}}$ be a stationary Markov chain (i.e. $X_0\sim\pi$) with transition
kernel $P$, and let $Z\in\llbracket k\rrbracket$ record the block
containing the initial state:
\[
    Z=i
    \quad\Longleftrightarrow\quad
    X_0\in O_i.
\]
We denote $\pi_i:=\pi(\,\cdot\mid O_i)$.  Then
\[
    \mathcal L(X_1)=\pi,
    \qquad
    \mathcal L(X_1\mid Z=i)=\pi_iP,
\]
and, for every $x\in O_i$,
$(G_{\mathcal O}P)(x,\cdot)=\pi_iP$.  The entrywise expression for
the weighted Frobenius norm therefore gives
\[
\begin{aligned}
    F(\mathcal O)
    &=\sum_{i=1}^k \pi(O_i)
      \sum_{y\in\X}
      \frac{\bigl((\pi_iP)(y)-\pi(y)\bigr)^2}{\pi(y)}\\
    &=\sum_{i=1}^k \pi(O_i)\,\chi^2(\pi_iP\|\pi).
\end{aligned}
\]
On the other hand,
$\mathbb P(Z=i,X_1=y)=\pi(O_i)(\pi_iP)(y)$, whereas the corresponding
product of marginals is $\pi(O_i)\pi(y)$.  Hence
\begin{equation}
    F(\mathcal O)
    =\sum_{i=1}^k \pi(O_i)\,\chi^2(\pi_iP\|\pi)
    =I_{\chi^2}(Z;X_1).
    \label{eq:F-chi-square-information}
\end{equation}
Thus $F(\mathcal O)$ is the residual Pearson chi-square information
that the state after one transition retains about the initial block
label.   In particular,
$F(\mathcal O)=0$ if and only if $Z$ and $X_1$ are independent.
Moreover, since $Z$ is determined by $X_0$,
$I_{\chi^2}(Z;X_0)=k-1$.  Combining this fact with
\eqref{eq:F-cut} gives
\[
    \operatorname{Ncut}_{P^2}(\mathcal O)
    =I_{\chi^2}(Z;X_0)-I_{\chi^2}(Z;X_1).
\]
Thus the normalized-cut term is exactly the amount of chi-square
information about the initial block that is lost after one
transition.  The identity \eqref{eq:F-chi-square-information} uses
stationarity but not reversibility; reversibility is used in the
trace and $P^2$ cut representations above.


\subsection{Eigenvalues and the spectral relaxation}\label{subsec:spectralrelax}

Let
\[
0\leq\eta_1\leq\eta_2\leq\cdots\leq\eta_{n-1}<1
\]
be the eigenvalues of $P^2$ on the centered subspace
$\ell^2_0(\pi)$ defined in \eqref{eq:centered-subspace}, listed in
nondecreasing order and counting multiplicity.  Choose corresponding
$\pi$-orthonormal eigenfunctions
$\phi_1,\ldots,\phi_{n-1}$.  Since $P$ is self-adjoint and commutes
with $P^2$, we choose these functions to be eigenfunctions of $P$ as
well.  Equivalently, the $\eta_j$ are the squares of the non-unit
eigenvalues of $P$, counted with multiplicity and rearranged in
nondecreasing order.

For $k\in\llbracket 2,n\rrbracket$, the variational principle for sums of
eigenvalues gives
\begin{equation}
	L_{k-1}:=\sum_{j=1}^{k-1}\eta_j
	=\min_{\substack{U\subset\ell^2_0(\pi)\\ \dim(U)=k-1}}
	\sum_{j=1}^{k-1}\langle u_j,P^2u_j\rangle_\pi,
	\label{eq:Ky-Fan}
\end{equation}
where $u_1,\ldots,u_{k-1}$ form an orthonormal basis of $U$.  For a
$k$-block partition $\mathcal O$, recall from
\eqref{eq:QO-VO} that $Q_{\mathcal O}$ projects onto the
$(k-1)$-dimensional subspace $\mathcal V_{\mathcal O}$.  If
$u_1,\ldots,u_{k-1}$ is an orthonormal basis of
$\mathcal V_{\mathcal O}$, then
\[
Q_{\mathcal O}
=\sum_{j=1}^{k-1}u_j\otimes_\pi u_j.
\]
It follows from the trace identity for $F$ that
\begin{align*}
	F(\mathcal O)
	&=\operatorname{Tr}(Q_{\mathcal O}P^2)\\
	&=\sum_{j=1}^{k-1}
	\operatorname{Tr}\bigl(
	(u_j\otimes_\pi u_j)P^2
	\bigr)\\
	&=\sum_{j=1}^{k-1}
	\langle u_j,P^2u_j\rangle_\pi.
\end{align*}
The minimum in \eqref{eq:Ky-Fan} is taken over all
$(k-1)$-dimensional subspaces of $\ell^2_0(\pi)$, whereas the
partition problem ranges only over the restricted family
$\{\mathcal V_{\mathcal O}:|\mathcal O|=k\}$.  Therefore
$L_{k-1}\leq F(\mathcal O)$ for every $k$-block partition
$\mathcal O$, and hence
\begin{equation}
	L_{k-1}\leq
	\min_{|\mathcal O|=k}F(\mathcal O).
	\label{eq:relaxation-lower-bound}
\end{equation}

\paragraph{Classical spectral clustering.}
For comparison, classical normalized spectral clustering uses the
bottom eigenfunctions of the Laplacian $I-P^2$, equivalently the
largest nonconstant eigenfunctions of $P^2$, to find blocks with small boundary.
In the two-block case, the eigenfunction is sorted and thresholded.  In
the multiway case, several eigenfunctions are used to embed the states,
after which a geometric rounding procedure is applied
\citep{NgJordanWeiss2001}.  Our problem has the opposite trace
orientation: it uses the smallest eigenvalues of $P^2$ and seeks blocks
with large two-step normalized boundary.

\section{The two-block case}
\label{sec:two-block}

We begin with the two-block problem, where the centered Gibbs
projection has rank one and the rounding problem reduces to
one-dimensional weighted two-means.  Section~\ref{sec:k-block}
extends the construction to general $k$.

\subsection{The rank-one representation and the relaxed problem}

Let $\varnothing\neq S\subsetneq\X$.  Define
\begin{equation*}
	z_S
	:=\sqrt{\frac{1-\pi(S)}{\pi(S)}}\,\1_S
	-\sqrt{\frac{\pi(S)}{1-\pi(S)}}\,\1_{S^c}.
\end{equation*}
Then
\[
\langle z_S,\1\rangle_\pi=0,
\qquad
\|z_S\|_\pi=1.
\]
The functions $\1$ and $z_S$ form an orthonormal basis of the space of
functions that are constant on $S$ and on $S^c$.  Since $G_S$ is the
orthogonal projection onto that space,
\begin{equation}
    G_S=\Pi+z_S\otimes_\pi z_S.
    \label{eq:GS-rank-one}
\end{equation}
This proves the rank-one identity directly.

Throughout this section, write $F(S):=F(\{S,S^c\})$.  Equations
\eqref{eq:F-trace} and \eqref{eq:GS-rank-one} give
\begin{equation*}
    F(S)
    =\langle z_S,P^2z_S\rangle_\pi
    =\|Pz_S\|_\pi^2.
\end{equation*}
Let
\[
    F_2^*:=\min_{\varnothing\neq S\subsetneq\X}F(S).
\]
Removing the restriction that the unit vector must have the form
$z_S$ and using \eqref{eq:relaxation-lower-bound} give
\[
\eta_1
=
\min_{\substack{f\in\ell^2_0(\pi)\\ \|f\|_\pi=1}}
\langle f,P^2f\rangle_\pi
\leq F_2^*.
\]
Let $\phi=\phi_1$ be a normalized eigenfunction associated with
$\eta_1$. The relaxed optimizer $\phi$ is a real-valued function rather than a
set, so the remaining task is to round its values to a two-block
partition.

\subsection{Rounding as weighted two-means}

For a nontrivial set $\varnothing\neq S\subsetneq\X$, define
\begin{equation}
    \delta(S):=\|(I-G_S)\phi\|_\pi^2,
    \qquad
    \delta_2(\phi):=
      \min_{\varnothing\neq S\subsetneq\X}\delta(S).
    \label{eq:delta-two}
\end{equation}
Because $\phi\perp\1$, equation \eqref{eq:GS-rank-one} implies
\begin{equation*}
    \delta(S)
    =1-\langle\phi,z_S\rangle_\pi^2.
\end{equation*}

Let
\[
    \overline\phi_S
    :=\frac{1}{\pi(S)}\sum_{x\in S}\pi(x)\phi(x),
    \qquad
    \overline\phi_{S^c}
    =\frac{1}{\pi(S^c)}\sum_{x\in S^c}\pi(x)\phi(x).
\]
Since $G_S\phi$ takes these two values,
\begin{equation}
\begin{split}
    \delta(S)
    ={}&\sum_{x\in S}\pi(x)
        \bigl(\phi(x)-\overline\phi_S\bigr)^2\\
       &+\sum_{x\in S^c}\pi(x)
        \bigl(\phi(x)-\overline\phi_{S^c}\bigr)^2.
\end{split}
    \label{eq:weighted-two-means}
\end{equation}
Thus $\delta_2(\phi)$ is exactly the weighted two-means error for the
one-dimensional points $\phi(x)$ with weights $\pi(x)$.

For a function $\phi:\X\to\mathbb R$, we call a nonempty proper set
$S\subsetneq\X$ a \emph{threshold set of $\phi$} if there exists
$\tau\in\mathbb R$ such that
\begin{equation*}
	\{x\in\X:\phi(x)<\tau\}
	\subseteq S
	\subseteq
	\{x\in\X:\phi(x)\leq\tau\}.
\end{equation*}
Thus all states whose $\phi$-values are strictly below the threshold
belong to $S$, all states whose values are strictly above it belong to
$S^c$, and states satisfying $\phi(x)=\tau$ may be assigned to either
side.

The one-dimensional structure ensures that restricting attention to
threshold sets does not increase the optimal rounding error.

\begin{lemma}[Threshold property]
\label{lem:threshold-property}
There is a minimizer of $\delta(S)$ that is a threshold set of $\phi$.
\end{lemma}

\begin{proof}
Take a partition minimizing \eqref{eq:weighted-two-means}, and denote
its two means by $m_1\leq m_2$.  They are distinct.  Indeed, if they
were equal, centeredness of $\phi$ would force both means to be zero
and the error would be one; choosing a singleton $\{x\}$ with
$\phi(x)\neq0$ gives a strictly smaller error.  Keeping the means fixed, assigning a
number $u$ to its nearest mean cannot increase the objective.  The
first mean is nearest exactly when
\[
    (u-m_1)^2\leq(u-m_2)^2,
\]
which is equivalent to
\[
    u\leq\frac{m_1+m_2}{2}.
\]
The resulting blocks are therefore separated by one threshold.
They are nonempty: a block with mean $m_1$ contains a value at most
$m_1$, and a block with mean $m_2$ contains a value at least $m_2$.
Recomputing their weighted means can only decrease the objective.
Since the original partition was optimal, the threshold partition is
also optimal.  Equal values may all be placed on either side of the
threshold.  This is the standard interval property of one-dimensional
$k$-means \citep{NielsenNock2014,SongZhong2020}.
\end{proof}

\subsection{The spectral sweep algorithm}

Order the states so that
\[
    \phi(x_{(1)})\leq\phi(x_{(2)})\leq\cdots
    \leq\phi(x_{(n)}),
\]
and define
\[
    S_j:=\{x_{(1)},\ldots,x_{(j)}\},
    \qquad j\in\llbracket n-1\rrbracket.
\]
The two-block spectral sweep is given in
Algorithm~\ref{alg:two-block-sweep}.

\begin{algorithm}[H]
\caption{Two-block spectral sweep}
\label{alg:two-block-sweep}
\begin{enumerate}
    \item Compute a normalized bottom eigenfunction $\phi$ of $P^2$, that is, choose any normalized bottom eigenfunction
    $\phi\in\ell^2_0(\pi)$ satisfying
    \[
    P^2\phi=\eta_1\phi,
    \qquad
    \|\phi\|_\pi=1.
    \]
    \item Sort the states by $\phi(x)$ and form the sets $S_j$.
    \item Evaluate the original objective $F(S_j)$ and return
    \[
        \widehat S\in
        \operatorname*{argmin}_{j\in\llbracket n-1\rrbracket}F(S_j).
    \]
    For later use in Section \ref{sec:numerical-experiments}, we write
    \[
    \widehat{\mathcal O}_2
    :=\{\widehat S,\widehat S^c\}.
    \]
\end{enumerate}
\end{algorithm}

Sorting is needed because it lists all one-dimensional Voronoi splits
of the values $\phi(x)$.  Lemma \ref{lem:threshold-property} implies
\begin{equation}
    \delta_2(\phi)=\min_{j\in\llbracket n-1\rrbracket}\delta(S_j).
    \label{eq:delta-sweep}
\end{equation}
A non-threshold set is still an admissible cut and may have a smaller
value of $F$.  The sweep does not claim otherwise.  Rather, threshold
sets are sufficient to obtain the following rounding guarantee.

\subsection{Guarantee for the two-block sweep}

\begin{theorem}[Two-block sweep guarantee]
\label{thm:two-block-guarantee}
Suppose $n\geq3$ and recall that $\delta_2$ is defined in \eqref{eq:delta-two}. The output $\widehat S$ of
Algorithm~\ref{alg:two-block-sweep} satisfies
\begin{equation}
\begin{split}
    \eta_1+(\eta_2-\eta_1)\delta_2(\phi)
    \leq F_2^*&\leq F(\widehat S)
    \leq
    \eta_1+(\eta_{n-1}-\eta_1)\delta_2(\phi).
\end{split}
    \label{eq:two-block-chain}
\end{equation}
Consequently,
\begin{equation}
    0\leq F(\widehat S)-F_2^*
    \leq(\eta_{n-1}-\eta_2)\delta_2(\phi).
    \label{eq:two-block-gap}
\end{equation}
If $\eta_2>\eta_1$, then also
\begin{equation}
    F(\widehat S)-\eta_1
    \leq
    \frac{\eta_{n-1}-\eta_1}{\eta_2-\eta_1}
    (F_2^*-\eta_1).
    \label{eq:two-block-relative}
\end{equation}
\end{theorem}

\begin{proof}
Choose a $\pi$-orthonormal eigenbasis
$\phi_1,\ldots,\phi_{n-1}$ of $P^2$ on $\ell^2_0(\pi)$
whose first vector is the eigenfunction $\phi$ selected by
Algorithm~\ref{alg:two-block-sweep}.
For a nontrivial $S$, expand
\[
    z_S=\sum_{j=1}^{n-1}c_j\phi_j,
    \qquad
    c_j=\langle z_S,\phi_j\rangle_\pi.
\]
Since $z_S$ is a unit vector, $\sum_jc_j^2=1$.  Moreover,
\[
    \delta(S)=1-c_1^2=\sum_{j=2}^{n-1}c_j^2.
\]
It follows that
\begin{align*}
    F(S)
    &=\sum_{j=1}^{n-1}\eta_jc_j^2\\
    &=\eta_1+
      \sum_{j=2}^{n-1}(\eta_j-\eta_1)c_j^2.
\end{align*}
Therefore, for every nontrivial $S$,
\begin{equation*}
    \eta_1+(\eta_2-\eta_1)\delta(S)
    \leq F(S)
    \leq
    \eta_1+(\eta_{n-1}-\eta_1)\delta(S).
\end{equation*}

The lower bound, minimized over all $S$, gives the first inequality in
\eqref{eq:two-block-chain}.  By \eqref{eq:delta-sweep}, a sweep set
attains $\delta_2(\phi)$.  Since $\widehat S$ has the smallest $F$
among all sweep sets,
\[
    F(\widehat S)
    \leq\eta_1+(\eta_{n-1}-\eta_1)\delta_2(\phi).
\]
This proves \eqref{eq:two-block-chain}.  Subtracting its lower bound
from its upper bound proves \eqref{eq:two-block-gap}.  Finally,
\[
    \delta_2(\phi)
    \leq\frac{F_2^*-\eta_1}{\eta_2-\eta_1}
\]
when $\eta_2>\eta_1$, and substitution in the upper bound proves
\eqref{eq:two-block-relative}.
\end{proof}

When $n=2$, there is only one nontrivial two-block partition, up to
complementation, so the sweep is exact.

\begin{corollary}[A two-valued eigenfunction]
\label{cor:two-valued}
If $\phi$ takes exactly two values, then
\[
    \delta_2(\phi)=0,
    \qquad
    F(\widehat S)=F_2^*=\eta_1.
\]
\end{corollary}

\begin{proof}
Let $S$ be one level set of $\phi$.  Then $\phi$ is constant on $S$
and on $S^c$, so $G_S\phi=\phi$ and $\delta(S)=0$.  Equivalently,
$\phi=\pm z_S$.  Hence
$F(S)=\langle\phi,P^2\phi\rangle_\pi=\eta_1$,
which meets the spectral lower bound.
\end{proof}

\subsection{Contrast with the classical Cheeger sweep}

For the two-step chain $P^2$, by recalling $\mathcal Q_{P^2}$ as defined in \eqref{eq:general-flow}, we define its conductance by
\[
    h_{P^2}(S)=
    \frac{\mathcal Q_{P^2}(S,S^c)}
         {\min\{\pi(S),\pi(S^c)\}},
    \qquad
    h_{P^2}^*=\min_{\varnothing\neq S\subsetneq\X}h_{P^2}(S).
\]
Let $\gamma_2$ be the second-smallest eigenvalue of the Laplacian
$I-P^2$, counting multiplicity.  The classical Cheeger sweep sorts a corresponding
low-frequency eigenfunction, evaluates the conductance of its
threshold sets, and returns a set $S_{\mathrm{Ch}}$ satisfying
\begin{equation*}
    \frac{\gamma_2}{2}
    \leq h_{P^2}^*
    \leq h_{P^2}(S_{\mathrm{Ch}})
    \leq\sqrt{2\gamma_2},
\end{equation*}
under the usual normalization \citep{LawlerSokal1988,SinclairJerrum1989}. Here ``low-frequency'' means associated with a small eigenvalue of
the Laplacian $I-P^2$. The proof uses a level-set, or co-area, argument to find a set with
small boundary.

Our sorting step looks similar, but the problem is different.  By
\eqref{eq:F-cut},
\begin{equation*}
    F(S)
    =1-
      \frac{\mathcal Q_{P^2}(S,S^c)}{\pi(S)\pi(S^c)}.
\end{equation*}
Thus minimizing $F$ maximizes, rather than minimizes, the two-way
normalized cut of $P^2$.  We use the bottom eigenfunction of $P^2$,
which is a ``high-frequency'' eigenfunction of $I-P^2$.  Sorting is used to
solve the one-dimensional two-means rounding problem, not to apply the
classical co-area proof.  

\section{The \texorpdfstring{$k$}{k}-block case}
\label{sec:k-block}

In this section, we fix $k\in\llbracket 2,n-1\rrbracket$.  The case $k=n$ has only the
singleton partition and requires no rounding.  Recall that $\phi_1,\ldots,\phi_{k-1}$ are the
$\pi$-orthonormal eigenfunctions associated with the $k-1$
smallest eigenvalues of $P^2$ on $\ell^2_0(\pi)$.  Define the raw
spectral embedding
\begin{equation}
    \Phi(x):=\bigl(\phi_1(x),\ldots,\phi_{k-1}(x)\bigr)
    \in\mathbb{R}^{k-1}.
    \label{eq:spectral-embedding}
\end{equation}
The constant eigenfunction is not included: a $k$-block Gibbs kernel
already contains the constants, so only $k-1$ nonconstant directions
must be rounded.

\subsection{The spectral rounding error and weighted
\texorpdfstring{$k$}{k}-means}

Let
\[
    \mathcal U_{k-1}:=\operatorname{span}\{\phi_1,\ldots,\phi_{k-1}\}.
\]
This is an optimal continuous subspace in the relaxation
\eqref{eq:Ky-Fan} and \eqref{eq:relaxation-lower-bound}.  A genuine $k$-block partition cannot produce an
arbitrary $(k-1)$-dimensional subspace.  It produces the centered
block-constant subspace $\mathcal V_{\mathcal O}$ defined in
\eqref{eq:QO-VO}.  Thus the discrete step replaces
the relaxed spectral subspace $\mathcal U_{k-1}$ by one of the feasible
subspaces $\mathcal V_{\mathcal O}$ indexed by partitions.  This is
the usual relaxation--discretization viewpoint behind spectral
rounding \citep{YuShi2003}; clustering the rows of a spectral embedding
by $k$-means is a standard implementation of that viewpoint
\citep{NgJordanWeiss2001,PengSunZanetti2017}.

Let
\[
E_{k-1}
:=\sum_{j=1}^{k-1}\phi_j\otimes_\pi\phi_j
\]
be the orthogonal projection onto
$\mathcal U_{k-1}$.  Recall that each state $x\in\X$ is represented
in the spectral embedding by the point
$\Phi(x)\in\mathbb R^{k-1}$.  For a partition
$\mathcal O=\{O_1,\ldots,O_k\}$ and each block $O_i$ with
$i\in\llbracket k\rrbracket$, define
\[
\mu_i
:=\frac{1}{\pi(O_i)}
\sum_{x\in O_i}\pi(x)\Phi(x)
\in\mathbb R^{k-1}.
\]
Thus $\mu_i$ is the $\pi$-weighted centroid of the embedded points
$\{\Phi(x):x\in O_i\}$.  Equivalently, its $j$th coordinate is the
stationary conditional mean of $\phi_j$ within $O_i$:
\[
(\mu_i)_j
=\frac{1}{\pi(O_i)}
\sum_{x\in O_i}\pi(x)\phi_j(x),
\qquad j\in\llbracket k-1\rrbracket.
\]
We call
\begin{equation}
    \delta(\mathcal O)
    :=\sum_{j=1}^{k-1}
      \|(I-G_{\mathcal O})\phi_j\|_\pi^2.
    \label{eq:delta-k}
\end{equation}
the \emph{spectral rounding error}.  This is a descriptive name used
here rather than a separate standard objective: the next result shows
that it measures precisely the loss incurred when the continuous
spectral solution is represented by a discrete partition.

\begin{proposition}[Two interpretations of the spectral rounding error]
\label{prop:kmeans-identity}
Let $k\in\llbracket 2,n-1\rrbracket$.
For every $k$-block partition $\mathcal O$,
\begin{equation}
    \delta(\mathcal O)
    =\sum_{i=1}^k\sum_{x\in O_i}
      \pi(x)\|\Phi(x)-\mu_i\|_2^2
    =\frac12\|E_{k-1}-Q_{\mathcal O}\|_{F,\pi}^2
    =\sum_{j=1}^{k-1}\sin^2\theta_j,
    \label{eq:kmeans-identity}
\end{equation}
where $\theta_1,\ldots,\theta_{k-1}$ are the principal angles between
$\mathcal U_{k-1}$ and $\mathcal V_{\mathcal O}$.
\end{proposition}

\begin{proof}
If $x\in O_i$, then $(G_{\mathcal O}\phi_j)(x)$ is the $j$th
coordinate of $\mu_i$.  Expand the left-hand side of
\eqref{eq:delta-k}, and sum the squared error first over the
coordinates and then over the states.  This gives
the first equality in \eqref{eq:kmeans-identity}.

For the subspace identities, every $\phi_j$ is centered, so
\[
    (I-G_{\mathcal O})\phi_j
    =(I-Q_{\mathcal O})\phi_j.
\]
Since $Q_{\mathcal O}$ is an orthogonal projection,
\[
    \delta(\mathcal O)
    =k-1-\operatorname{Tr}(E_{k-1}Q_{\mathcal O}).
\]
Both $E_{k-1}$ and $Q_{\mathcal O}$ have rank $k-1$.  Expanding their
Hilbert--Schmidt distance therefore gives
\[
    \|E_{k-1}-Q_{\mathcal O}\|_{F,\pi}^2
    =2(k-1)-2\operatorname{Tr}(E_{k-1}Q_{\mathcal O})
    =2\delta(\mathcal O).
\]
Finally, the singular values of the cross-Gram matrix formed from
orthonormal bases of $\mathcal U_{k-1}$ and $\mathcal V_{\mathcal O}$ are
$\cos\theta_1,\ldots,\cos\theta_{k-1}$.  Hence
\[
    \operatorname{Tr}(E_{k-1}Q_{\mathcal O})
    =\sum_{j=1}^{k-1}\cos^2\theta_j,
\]
which proves the last identity.
\end{proof}

The first equality in \eqref{eq:kmeans-identity} is the
$\pi$-weighted within-cluster sum of squares of the spectral points
$\Phi(x)$.  Equivalently, it is the expected squared quantization
error when $\Phi(x)$ is replaced by the centroid assigned to $x$.
Indeed,
\[
    \sum_{x\in\X}\pi(x)\Phi(x)=0,
    \qquad
    \sum_{x\in\X}\pi(x)\|\Phi(x)\|_2^2=k-1.
\]
Thus minimizing $\delta(\mathcal O)$ is exactly weighted $k$-means on
the raw embedding, with state weight $\pi(x)$.

The remaining equalities give the complementary geometric meaning.
The quantity $\delta(\mathcal O)$ is the squared chordal distance
between $\mathcal U_{k-1}$ and $\mathcal V_{\mathcal O}$, or one half of
the squared Hilbert--Schmidt distance between their projection
operators.  Here squared chordal distance means
$\sum_{j=1}^{k-1}\sin^2\theta_j$.  In particular,
\[
    \delta(\mathcal O)=0
    \quad\Longleftrightarrow\quad
    \mathcal U_{k-1}=\mathcal V_{\mathcal O}.
\]
This explains all three words in the name: the subspace
$\mathcal U_{k-1}$ is \emph{spectral}, a partition \emph{rounds} it to a
feasible block-constant subspace, and $\delta$ measures the resulting
\emph{error}.  The value is unchanged by an orthogonal change of
basis within a fixed $\mathcal U_{k-1}$.  If
$\eta_{k-1}<\eta_{k}$, that bottom subspace is canonical; if the cutoff
eigenvalue is tied, $\delta$ refers to the particular optimal
$(k-1)$-dimensional subspace selected for the embedding.

\subsection{The multiway spectral sweep}

The ideal rounding step globally minimizes $\delta(\mathcal O)$.  In
practice, weighted $k$-means is usually run from several
initializations.  The resulting procedure is given in
Algorithm~\ref{alg:multiway-sweep}.

\begin{algorithm}[H]
\caption{Multiway spectral sweep for $F$}
\label{alg:multiway-sweep}
\begin{enumerate}
    \item Compute $\phi_j$ for
    $j\in\llbracket k-1\rrbracket$ and form $\Phi(x)$.
    \item \emph{Spectral rounding.} Approximately minimize the spectral rounding error
    $\delta(\mathcal O)$ by running weighted $k$-means on the embedded
    points $\Phi(x)$, with weights $\pi(x)$, from several
    initializations.  In view of Proposition~\ref{prop:kmeans-identity},
    the weighted $k$-means distortion is exactly $\delta(\mathcal O)$.
    Retain the terminal partitions having $k$ nonempty blocks and denote
    them by
    $\mathcal O^{(1)},\ldots,\mathcal O^{(\ell)}$.
    \item \emph{Candidate rescoring.} Evaluate the original objective $F$ on every candidate and
    return
    \[
        \widehat{\mathcal O}_F\in
        \operatorname*{argmin}_{s\in\llbracket\ell\rrbracket}
        F(\mathcal O^{(s)}).
    \]
\end{enumerate}
\end{algorithm}

We minimize $\delta$ to generate candidates close to the relaxed
bottom spectral subspace, but we use $F$ to choose the final candidate.  In
general,
\[
    \operatorname*{argmin}_{\mathcal O}\delta(\mathcal O)
    \neq
    \operatorname*{argmin}_{\mathcal O}F(\mathcal O).
\]
The next result quantifies how close the two objectives are.

\subsection{Guarantee for the multiway sweep}

Write
\begin{equation}
    g:=\eta_{k}-\eta_{k-1},
    \qquad
    W:=\eta_{n-1}-\eta_1.
    \label{eq:g-W}
\end{equation}
Here $g$ is the eigengap after the bottom $k-1$ eigenvalues and $W$ is
the width of the nonconstant spectrum; the relaxed optimum
$L_{k-1}$ was defined in \eqref{eq:Ky-Fan}.

For a partition $\mathcal O$, define its projection mass on the
$j$th eigenfunction by
\begin{equation}
    a_j(\mathcal O)
    :=\|Q_{\mathcal O}\phi_j\|_\pi^2,
    \qquad j\in\llbracket n-1\rrbracket.
    \label{eq:spectral-projection-mass}
\end{equation}
Because $Q_{\mathcal O}$ is an orthogonal projection of rank $k-1$,
\begin{equation*}
    0\leq a_j(\mathcal O)\leq1,
    \qquad
    \sum_{j=1}^{n-1}a_j(\mathcal O)=k-1.
\end{equation*}
Moreover, \eqref{eq:delta-k} gives the mass-balance identity
\begin{equation}
    \sum_{j=k}^{n-1}a_j(\mathcal O)
    =\sum_{j=1}^{k-1}\bigl(1-a_j(\mathcal O)\bigr)
    =\delta(\mathcal O).
    \label{eq:mass-balance}
\end{equation}

We first relate $F(\mathcal O)$ to the rounding error and these
spectral quantities.
\begin{theorem}[Spectral sandwich]
\label{thm:k-sandwich}
Let $k\in\llbracket 2,n-1\rrbracket$, and recall that $g, W$ are introduced in \eqref{eq:g-W} and $L_{k-1}$ is defined in \eqref{eq:Ky-Fan}. For every $k$-block partition $\mathcal O$,
\begin{equation}
    L_{k-1}+g\,\delta(\mathcal O)
    \leq F(\mathcal O)
    \leq L_{k-1}+W\,\delta(\mathcal O).
    \label{eq:k-sandwich}
\end{equation}
\end{theorem}

\begin{proof}
The spectral representation of \eqref{eq:F-trace} is
\[
    F(\mathcal O)
    =\sum_{j=1}^{n-1}\eta_ja_j(\mathcal O).
\]
Hence
\[
    F(\mathcal O)-L_{k-1}
    =\sum_{j=k}^{n-1}\eta_ja_j(\mathcal O)
     -\sum_{j=1}^{k-1}\eta_j
       \bigl(1-a_j(\mathcal O)\bigr).
\]
Both sums in \eqref{eq:mass-balance} have total mass
$\delta(\mathcal O)$.  For the lower bound, use
$\eta_j\geq\eta_{k}$ in the first sum and
$\eta_j\leq\eta_{k-1}$ in the second.  For the upper bound, use
$\eta_j\leq\eta_{n-1}$ in the first sum and
$\eta_j\geq\eta_1$ in the second.  This gives
\eqref{eq:k-sandwich}.
\end{proof}

We note that $\delta(\mathcal O)$ is a geometric surrogate and should not
be confused with an actual suboptimality gap for $F$.  The
spectral sandwich in Theorem~\ref{thm:k-sandwich} shows why the
surrogate is useful: with $g$ and $W$ as in \eqref{eq:g-W},
\[
g\,\delta(\mathcal O)
\leq F(\mathcal O)-L_{k-1}
\leq W\,\delta(\mathcal O).
\]
Thus $\delta$ measures unweighted projection mass leaking away from
the relaxed bottom spectral subspace, while
$F-L_{k-1}$ is the eigenvalue-weighted cost of moving that mass from
bottom modes to higher modes.

Let
\[
    F_k^*:=\min_{|\mathcal O|=k}F(\mathcal O),
    \qquad
    \delta_k^*:=\min_{|\mathcal O|=k}\delta(\mathcal O),
\]
and define
\[
    \overline\delta_\ell
    :=\min_{s\in\llbracket\ell\rrbracket}
      \delta(\mathcal O^{(s)}).
\]

Applying Theorem~\ref{thm:k-sandwich} to the best generated candidate
gives the following algorithmic guarantee.

\begin{corollary}[Algorithmic guarantee]
\label{cor:k-algorithm}
Let $k\in\llbracket 2,n-1\rrbracket$.
The output of Algorithm~\ref{alg:multiway-sweep} satisfies
\begin{equation}
    0\leq F(\widehat{\mathcal O}_F)-F_k^*
    \leq W\overline\delta_\ell-g\delta_k^*.
    \label{eq:k-data-bound}
\end{equation}
If at least one candidate is a $\rho$-approximation for weighted
$k$-means, so that
$\overline\delta_\ell\leq\rho\delta_k^*$, then
\begin{equation*}
    F(\widehat{\mathcal O}_F)-F_k^*
    \leq(\rho W-g)\delta_k^*.
\end{equation*}
In particular, every exact minimizer $\mathcal O$ of $\delta$
satisfies
\begin{equation*}
    F(\mathcal O)-F_k^*
    \leq(W-g)\delta_k^*.
\end{equation*}
\end{corollary}

\begin{proof}
Theorem \ref{thm:k-sandwich}, minimized over all partitions, gives
\[
    F_k^*\geq L_{k-1}+g\delta_k^*.
\]
Choose a generated candidate attaining $\overline\delta_\ell$.
Since the algorithm chooses the candidate with the smallest value of
$F$,
\[
    F(\widehat{\mathcal O}_F)
    \leq L_{k-1}+W\overline\delta_\ell.
\]
Subtracting the two inequalities proves \eqref{eq:k-data-bound}.
The other two claims follow by substituting
$\overline\delta_\ell\leq\rho\delta_k^*$ and $\rho=1$,
respectively.
\end{proof}

There is also the computable, although sometimes weaker, certificate
\begin{equation*}
    F(\widehat{\mathcal O}_F)-F_k^*
    \leq F(\widehat{\mathcal O}_F)-L_{k-1}.
\end{equation*}
If $\delta_k^*=0$, then the bottom spectral subspace is exactly
block-constant for some partition, and both the ideal spectral
rounding and the original problem attain $L_{k-1}$.

\subsection{Contrast with classical multiway spectral clustering}

Equation \eqref{eq:F-cut} shows that our problem is the reverse
normalized-cut problem for $P^2$:
we maximize cross-block flow.  Classical spectral clustering instead
minimizes normalized cut, using the largest nonconstant eigenvalues of
$P^2$, or equivalently the smallest nonzero eigenvalues of $I-P^2$
\citep{ShiMalik2000,NgJordanWeiss2001}.  Those eigenfunctions describe
slowly varying, nearly invariant regions.  Our bottom eigenfunctions
describe directions that $P$ suppresses most strongly.

There are also two differences in the rounding step.  First, the
identity \eqref{eq:kmeans-identity} uses the raw embedding $\Phi$ and
the weights $\pi(x)$; row normalization, often used in classical
spectral clustering, does not in general preserve this exact identity.  Second,
classical higher-order Cheeger guarantees are proved using geometric
random partitions and local scalar sweeps, not by an analysis of
ordinary $k$-means \citep{LeeOveisGharanTrevisan2014}.  They give
square-root-type conductance bounds.  Theorem \ref{thm:k-sandwich} is
different: it gives a linear transfer from the weighted $k$-means
error $\delta$ to the Frobenius objective $F$, governed by the
eigengap $g$ and spectral width $W$.

\section{Additive mixtures and a \texorpdfstring{$P$}{P}-spectral
multiway sweep}
\label{sec:additive-mixture}

We now apply the same relaxation-and-rounding strategy to the additive
mixture
\[
A_{\mathcal O}:=\frac{P+G_{\mathcal O}}{2}.
\]
Fix $k\in\llbracket 2,n-1\rrbracket$.  Since both $P$
and $G_{\mathcal O}$ are $\pi$-reversible Markov kernels,
$A_{\mathcal O}$ is also a $\pi$-reversible Markov kernel.

One motivation and key difference for studying this additive mixture is that its
partition-dependent spectral term involves $P$ itself, rather than
$P^2$.  Consequently, the relaxation and spectral embedding below
use the bottom eigenfunctions of $P$, rather than the bottom
eigenfunctions of $P^2$ used for the objective $F$ in
Section~\ref{sec:k-block}.  This distinction preserves the signs of
the eigenvalues of $P$ and can be important when $P$ has negative
eigenvalues.

\subsection{The mixture objective and its cut interpretation}

Consider the squared weighted Frobenius distance of the additive
mixture kernel from stationarity:
\begin{equation}
	H(\mathcal O)
	:=\|A_{\mathcal O}-\Pi\|_{F,\pi}^2.
	\label{eq:H-definition}
\end{equation}
The next proposition derives a trace representation of this objective
and isolates its partition-dependent part.

\begin{proposition}[Additive-mixture identity]
	\label{prop:additive-mixture-identity}
	Let $k\in\llbracket 2,n-1\rrbracket$.
	For every $k$-block partition $\mathcal O$,
	\begin{equation}
		H(\mathcal O)
		=\operatorname{Tr}\left(
		\frac{P^2}{4}
		+\frac{G_{\mathcal O}P}{2}
		+\frac{G_{\mathcal O}}{4}
		-\Pi
		\right).
		\label{eq:H-trace}
	\end{equation}
	Moreover,
	\begin{equation}
		H(\mathcal O)
		=C_{P,k}
		+\frac12\operatorname{Tr}(Q_{\mathcal O}P),
		\label{eq:H-reduction}
	\end{equation}
	where
	\begin{equation}
		C_{P,k}
		:=\frac14\operatorname{Tr}(P^2)+\frac{k-2}{4}.
		\label{eq:H-constant}
	\end{equation}
	Consequently, for fixed $P$ and $k$,
	\begin{equation*}
		\operatorname*{argmin}_{|\mathcal O|=k}H(\mathcal O)
		=
		\operatorname*{argmin}_{|\mathcal O|=k}
		\operatorname{Tr}(Q_{\mathcal O}P).
	\end{equation*}
\end{proposition}

\begin{proof}
	Since $P$ and $G_{\mathcal O}$ are self-adjoint,
	$A_{\mathcal O}$ is self-adjoint.  Moreover,
	\[
	A_{\mathcal O}\Pi
	=\Pi A_{\mathcal O}
	=\Pi.
	\]
	It follows that
	\[
	\begin{aligned}
		H(\mathcal O)
		&=\operatorname{Tr}\left(
		(A_{\mathcal O}-\Pi)^*
		(A_{\mathcal O}-\Pi)
		\right) \\
		&=\operatorname{Tr}\left(
		(A_{\mathcal O}-\Pi)^2
		\right) \\
		&=\operatorname{Tr}(A_{\mathcal O}^2-\Pi).
	\end{aligned}
	\]
	Because
	\[
	A_{\mathcal O}
	=\frac{P+G_{\mathcal O}}{2}
	\]
	and $G_{\mathcal O}^2=G_{\mathcal O}$, we have
	\[
	A_{\mathcal O}^2
	=\frac14\left(
	P^2+PG_{\mathcal O}
	+G_{\mathcal O}P+G_{\mathcal O}
	\right).
	\]
	Cyclicity of the trace gives
	\[
	\operatorname{Tr}(PG_{\mathcal O})
	=\operatorname{Tr}(G_{\mathcal O}P).
	\]
	Substituting into the preceding expression proves
	\eqref{eq:H-trace}.  Notice that this is an equality of traces;
	the corresponding operators need not be equal because $P$ and
	$G_{\mathcal O}$ need not commute.
	
	It remains to isolate the partition-dependent term.  Since
	\[
	G_{\mathcal O}=\Pi+Q_{\mathcal O},
	\]
	we have
	\[
	\operatorname{Tr}(G_{\mathcal O}P)
	=\operatorname{Tr}(\Pi P)
	+\operatorname{Tr}(Q_{\mathcal O}P)
	=1+\operatorname{Tr}(Q_{\mathcal O}P).
	\]
	Furthermore, $G_{\mathcal O}$ is a rank-$k$ orthogonal projection, so
	\[
	\operatorname{Tr}(G_{\mathcal O})=k.
	\]
	Therefore, \eqref{eq:H-trace} becomes
	\[
	\begin{aligned}
		H(\mathcal O)
		&=\frac14\operatorname{Tr}(P^2)
		+\frac12\left(
		1+\operatorname{Tr}(Q_{\mathcal O}P)
		\right)
		+\frac{k}{4}-1 \\
		&=\frac14\operatorname{Tr}(P^2)
		+\frac{k-2}{4}
		+\frac12\operatorname{Tr}(Q_{\mathcal O}P),
	\end{aligned}
	\]
	which proves \eqref{eq:H-reduction}.  The first two terms depend only
	on $P$ and $k$, so the equality of the two sets of minimizers follows.
\end{proof}

The general trace--cut identity \eqref{eq:general-trace-cut}, applied
to $B=P$, gives
\begin{equation}
    \operatorname{Tr}(Q_{\mathcal O}P)
    =k-1-\operatorname{Ncut}_P(\mathcal O).
    \label{eq:QP-cut}
\end{equation}
Thus, for fixed $k$, minimizing $H$ is equivalent to maximizing the
multiway normalized cut of the one-step chain $P$.

\subsection{The bottom spectral subspace of
	\texorpdfstring{$P$}{P}}

Let
\[
-1<\lambda_1\leq\lambda_2\leq\cdots
\leq\lambda_{n-1}<1
\]
be the eigenvalues of $P$ on $\ell^2_0(\pi)$, listed in
nondecreasing order and counting multiplicity (recall that $P$ is assumed to be ergodic and reversible).  Choose corresponding
$\pi$-orthonormal eigenfunctions
$\psi_1,\ldots,\psi_{n-1}$.  Define
\begin{equation*}
	E_{k-1}^P
	:=\sum_{j=1}^{k-1}\psi_j\otimes_\pi\psi_j.
\end{equation*}
Then $E_{k-1}^P$ is the orthogonal projection onto the selected
bottom $(k-1)$-dimensional spectral subspace
\[
\operatorname{span}\{\psi_1,\ldots,\psi_{k-1}\}
\subseteq\ell^2_0(\pi).
\]
The Ky Fan variational principle \citep{Fan1951} and
\eqref{eq:H-reduction} give
\begin{equation}
	L_H
	:=C_{P,k}+\frac12\sum_{j=1}^{k-1}\lambda_j
	\leq
	H_k^*
	:=\min_{|\mathcal O|=k}H(\mathcal O).
	\label{eq:H-relaxed-lower-bound}
\end{equation}
The corresponding raw spectral embedding is
\begin{equation}
	\Psi(x)
	=\bigl(\psi_1(x),\ldots,\psi_{k-1}(x)\bigr)
	\in\mathbb R^{k-1}.
	\label{eq:P-spectral-embedding}
\end{equation}
If $\lambda_{k-1}<\lambda_k$, the spectral subspace
$\operatorname{Ran}(E_{k-1}^P)$ and its orthogonal projector are
uniquely determined. If $\lambda_{k-1}=\lambda_k$, a repeated eigenspace straddles the
spectral cutoff.  The selected $(k-1)$-dimensional spectral subspace,
and hence $E_{k-1}^P$ and $\Psi$, need not be unique.  Throughout, we
fix one such choice.  Every choice has the same relaxed value $L_H$,
and the subsequent bounds apply to the selected projector and
embedding.

For a $k$-block partition, define the analogue of the spectral
rounding error by
\begin{equation}
    \delta_P(\mathcal O)
    :=\sum_{j=1}^{k-1}
      \|(I-G_{\mathcal O})\psi_j\|_\pi^2.
    \label{eq:delta-P}
\end{equation}
If
\[
    \nu_i
    =\frac1{\pi(O_i)}
      \sum_{x\in O_i}\pi(x)\Psi(x),
\]
then the same projection calculation as in
Proposition~\ref{prop:kmeans-identity} gives
\begin{equation}
\begin{split}
    \delta_P(\mathcal O)
    &=\sum_{i=1}^k\sum_{x\in O_i}
      \pi(x)\|\Psi(x)-\nu_i\|_2^2\\
    &=\frac12\|E_{k-1}^P-Q_{\mathcal O}\|_{F,\pi}^2.
\end{split}
    \label{eq:delta-P-identities}
\end{equation}
Indeed, $G_{\mathcal O}\psi_j$ is the blockwise centroid of the
$j$th coordinate, which proves the weighted $k$-means identity.
Since each $\psi_j$ is centered,
\[
    \delta_P(\mathcal O)
    =k-1-\operatorname{Tr}(E_{k-1}^P Q_{\mathcal O}).
\]
Both $E_{k-1}^P$ and $Q_{\mathcal O}$ are orthogonal projections of
rank $k-1$, and expanding their squared Frobenius distance proves
the second identity in \eqref{eq:delta-P-identities}.  Therefore
$\delta_P$ is simultaneously the $\pi$-weighted $k$-means
distortion of the embedding $\Psi$ and the squared chordal distance
between the relaxed bottom spectral subspace and the centered
block-constant subspace.

\subsection{The additive-mixture multiway spectral sweep}

The practical general-$k$ procedure is given in
Algorithm~\ref{alg:additive-mixture-sweep}.

\begin{algorithm}[H]
\caption{Additive-mixture multiway spectral sweep}
\label{alg:additive-mixture-sweep}
\begin{enumerate}
    \item Compute the eigenfunctions $\psi_j$ for
    $j\in\llbracket k-1\rrbracket$ from the bottom
    nonconstant spectral subspace of $P$, and form the embedding $\Psi$ in
    \eqref{eq:P-spectral-embedding}.
    \item \emph{Spectral rounding.} Approximately minimize the spectral rounding error
    $\delta_P(\mathcal O)$ by running weighted $k$-means on the embedded
    points $\Psi(x)$, with weights $\pi(x)$, from several
    initializations.  In view of \eqref{eq:delta-P-identities}, the
    weighted $k$-means distortion is exactly $\delta_P(\mathcal O)$.
    Retain the terminal partitions having $k$ nonempty blocks and denote
    them by
    $\mathcal O^{(1)},\ldots,\mathcal O^{(\ell)}$.
    \item \emph{Candidate rescoring.} Evaluate the original additive-mixture objective $H$ for
    every candidate and return
    \[
        \widehat{\mathcal O}_H
        \in\operatorname*{argmin}_{s\in\llbracket\ell\rrbracket}
          H(\mathcal O^{(s)}).
    \]
\end{enumerate}
\end{algorithm}

The weighted $k$-means stage generates partitions with small
$\delta_P$, while the final choice is made using $H$, because a
partition minimizing the geometric surrogate $\delta_P$ need not
minimize the original additive-mixture objective.  By
\eqref{eq:H-reduction} and \eqref{eq:QP-cut}, the final step could
equivalently select the candidate with the smallest
$\operatorname{Tr}(Q_{\mathcal O}P)$ or the largest
$\operatorname{Ncut}_P(\mathcal O)$.

\subsection{Guarantee for the additive-mixture sweep}

Define
\begin{equation}
    g_P:=\lambda_{k}-\lambda_{k-1},
    \qquad
    W_P:=\lambda_{n-1}-\lambda_1.
    \label{eq:gP-WP}
\end{equation}

We now state a spectral sandwich result analogous to Theorem \ref{thm:k-sandwich} but in the context of additive-mixture:

\begin{theorem}[Additive-mixture spectral sandwich]
\label{thm:H-sandwich}
Let $k\in\llbracket 2,n-1\rrbracket$, and recall that $g_P, W_P$ are introduced in \eqref{eq:gP-WP}, while $L_H$, $\delta_P$ are defined in \eqref{eq:H-relaxed-lower-bound} and \eqref{eq:delta-P} respectively. For every $k$-block partition $\mathcal O$,
\begin{equation}
    L_H+\frac{g_P}{2}\delta_P(\mathcal O)
    \leq H(\mathcal O)
    \leq
    L_H+\frac{W_P}{2}\delta_P(\mathcal O).
    \label{eq:H-sandwich}
\end{equation}
\end{theorem}

\begin{proof}
For $j\in\llbracket n-1\rrbracket$, set
\[
    b_j=\|Q_{\mathcal O}\psi_j\|_\pi^2.
\]
The projection $Q_{\mathcal O}$ has rank $k-1$ on $\ell^2_0(\pi)$, so
$\sum_{j=1}^{n-1}b_j=k-1$.  It follows that
\begin{equation}
    \delta_P(\mathcal O)
    =\sum_{j=1}^{k-1}(1-b_j)
    =\sum_{j=k}^{n-1}b_j.
    \label{eq:H-mass-balance}
\end{equation}
Moreover,
\[
    \operatorname{Tr}(Q_{\mathcal O}P)
    =\sum_{j=1}^{n-1}\lambda_jb_j.
\]
Using \eqref{eq:H-reduction} and
\eqref{eq:H-relaxed-lower-bound}, we obtain
\[
\begin{split}
    H(\mathcal O)-L_H
    =\frac12\biggl(
       \sum_{j=k}^{n-1}\lambda_jb_j
       -\sum_{j=1}^{k-1}\lambda_j(1-b_j)
      \biggr).
\end{split}
\]
For the lower bound, use
$\lambda_j\geq\lambda_{k}$ in the first sum and
$\lambda_j\leq\lambda_{k-1}$ in the second.  For the upper bound, use
$\lambda_j\leq\lambda_{n-1}$ in the first sum and
$\lambda_j\geq\lambda_1$ in the second.  The mass-balance identity
\eqref{eq:H-mass-balance} then proves \eqref{eq:H-sandwich}.
\end{proof}

Let
\[
    \delta_{P,k}^*
    =\min_{|\mathcal O|=k}\delta_P(\mathcal O),
    \qquad
    \overline\delta_{P,\ell}
    =\min_{s\in\llbracket\ell\rrbracket}
      \delta_P(\mathcal O^{(s)}).
\]

Analogous to Corollary \ref{cor:k-algorithm}, we make use of the spectral sandwich result Theorem \ref{thm:H-sandwich} to give a guarantee of Algorithm~\ref{alg:additive-mixture-sweep}:

\begin{corollary}[Algorithmic guarantee for the mixture objective]
\label{cor:H-algorithm}
Let $k\in\llbracket 2,n-1\rrbracket$.
The output of Algorithm~\ref{alg:additive-mixture-sweep} satisfies
\begin{equation}
    0\leq H(\widehat{\mathcal O}_H)-H_k^*
    \leq\frac12\left(
       W_P\overline\delta_{P,\ell}
       -g_P\delta_{P,k}^*
      \right).
    \label{eq:H-data-bound}
\end{equation}
If at least one candidate is a $\rho$-approximation for weighted
$k$-means, so that
$\overline\delta_{P,\ell}\leq\rho\delta_{P,k}^*$, then
\begin{equation}
    H(\widehat{\mathcal O}_H)-H_k^*
    \leq\frac12(\rho W_P-g_P)\delta_{P,k}^*.
    \label{eq:H-rho-bound}
\end{equation}
\end{corollary}

\begin{proof}
Minimizing the lower bound in Theorem~\ref{thm:H-sandwich} over all
$k$-block partitions gives
\[
    H_k^*\geq L_H+\frac{g_P}{2}\delta_{P,k}^*.
\]
Choose a generated candidate attaining
$\overline\delta_{P,\ell}$.  Since
$\widehat{\mathcal O}_H$ has the smallest $H$ among all generated
candidates, the upper bound in Theorem~\ref{thm:H-sandwich} gives
\[
    H(\widehat{\mathcal O}_H)
    \leq L_H+\frac{W_P}{2}\overline\delta_{P,\ell}.
\]
Subtracting proves \eqref{eq:H-data-bound}; substituting the
$\rho$-approximation assumption proves \eqref{eq:H-rho-bound}.
\end{proof}

The computable lower bound $L_H$ also gives the certificate
\[
    H(\widehat{\mathcal O}_H)-H_k^*
    \leq H(\widehat{\mathcal O}_H)-L_H.
\]
If $\delta_{P,k}^*=0$, the selected bottom spectral subspace of $P$ is
exactly a centered block-constant subspace, and both the ideal
rounding and the original objective attain $L_H$.

\subsection{Relation to the earlier objective and to classical
spectral clustering}

The earlier objective $F$ depends on $P^2$, so its bottom embedding
uses the eigenvalues of $P$ having the smallest absolute values.  In
contrast, $H$ uses the algebraically smallest eigenvalues of $P$.
Thus the additive-mixture relaxation is sensitive to negative modes:
it seeks a block-constant subspace aligned with directions on which
$P$ has the most negative action.  If $P$ is positive semidefinite,
the bottom spectral subspaces of $P$ and $P^2$ agree, up to choices inside
tied eigenspaces; otherwise the two embeddings may be very different.

Equation \eqref{eq:QP-cut} also shows that this remains a reverse
normalized-cut problem.  Ordinary normalized spectral clustering
seeks small cross-block flow and uses slowly varying eigenfunctions.
Here one seeks large one-step cross-block flow and uses the bottom
eigenfunctions of $P$, in the same broad smallest-eigenvalue
orientation that appears in spectral approaches to maximum cut
\citep{Trevisan2012}.  The continuous relaxation is rounded to a
partition by weighted $k$-means, following the general
relaxation--discretization principle of multiclass spectral
clustering \citep{YuShi2003}.

\section{Multi-horizon distance to stationarity}
\label{sec:multi-horizon}

The objective $F(\mathcal O)$ in \eqref{eq:F-definition} measures the
distance to stationarity after one averaging step $G_\mathcal{O}$ followed by one
application of $P$.  A partition that performs well at one step need
not be the best partition over a longer time scale.  This motivates
aggregating the distance to stationarity by considering powers of Markov kernels.  Using powers of
a Markov kernel to describe the effects of several time scales is also a
standard idea in multiscale diffusion geometry
\citep{CoifmanLafon2006}.  In the present context, the aggregation
retains the exact trace structure needed for spectral relaxation.

Throughout this section, fix
$k\in\llbracket 2,n-1\rrbracket$.

\subsection{Lagged and aggregated objectives}

For a positive integer $t\geq1$, define the lag-$t$ objective
\begin{equation}
    F(P^t,\mathcal O)
    :=\|G_{\mathcal O}P^t-\Pi\|_{F,\pi}^2.
    \label{eq:F-lag-t}
\end{equation}
With the stationary Markov chain $(X_t)_{t \in \mathbb{N}\cup\{0\}}$, initial block label $Z$,
block masses $\pi(O_i)$, and conditional laws $\pi_i$ introduced in the
discussion of \eqref{eq:F-chi-square-information} for $i \in \llbracket k \rrbracket$, the same
calculation with $P^t$ in place of $P$ gives
\begin{equation*}
    F(P^t,\mathcal O)
    =\sum_{i=1}^k \pi(O_i)\,\chi^2(\pi_iP^t\|\pi)
    =I_{\chi^2}(Z;X_t),
    \qquad t\geq1.
    \label{eq:F-lag-chi-square-information}
\end{equation*}
Indeed, $\mathcal L(X_t\mid Z=i)=\pi_iP^t$ and
$\mathcal L(X_t)=\pi$.  Thus the lag-$t$ objective measures the
Pearson chi-square information about the initial block label that
remains after $t$ transitions.  The finite- and discounted
multi-horizon objectives defined below therefore aggregate this
residual block information over several lags.  

Thus $F(P,\mathcal O)=F(\mathcal O)$.  Given a finite horizon
$T\geq1$ and weights $w_1,\ldots,w_T\geq0$ that are not all zero,
define
\begin{equation}
    F_w^T(\mathcal O)
    :=\sum_{t=1}^T w_t F(P^t,\mathcal O).
    \label{eq:F-wT}
\end{equation}
For a discount factor $0\leq\gamma<1$, define the discounted
infinite-horizon objective
\begin{equation}
    F_\gamma^\infty(\mathcal O)
    :=\sum_{t=1}^\infty
      \gamma^{t-1}F(P^t,\mathcal O).
    \label{eq:F-gamma-infinity}
\end{equation}
Multiplying all the weights by the same positive constant does not
change the minimizing partitions.  Consequently, the weights in
\eqref{eq:F-wT} may be normalized to sum to one when a time-average
interpretation is desired.  Similarly, $(1-\gamma)F_\gamma^\infty$
and $F_\gamma^\infty$ have the same minimizers.

The next proposition gives trace formulae for
$F(P^t,\mathcal O)$, $F_w^T$, and $F_\gamma^\infty$.

\begin{proposition}[Trace and distance to stationarity identities]
\label{prop:multi-horizon-identities}
Let $k\in\llbracket 2,n-1\rrbracket$.
For every $t\geq1$ and every $k$-block partition $\mathcal O$,
\begin{equation}
    F(P^t,\mathcal O)
    =\operatorname{Tr}(Q_{\mathcal O}P^{2t}).
    \label{eq:F-lag-trace}
\end{equation}
Consequently,
\begin{align}
    F_w^T(\mathcal O)
    &=\operatorname{Tr}(Q_{\mathcal O}B_w^T),
    &
    B_w^T&:=\sum_{t=1}^T w_tP^{2t},
    \label{eq:F-wT-trace}\\
    F_\gamma^\infty(\mathcal O)
    &=\operatorname{Tr}(Q_{\mathcal O}B_\gamma^\infty),
    &
    B_\gamma^\infty
    &:=P^2(I-\gamma P^2)^{-1}.
    \label{eq:F-gamma-trace}
\end{align}

\end{proposition}

\begin{proof}
Since $P^t\Pi=\Pi P^t=\Pi$,
\[
    G_{\mathcal O}P^t-\Pi=Q_{\mathcal O}P^t.
\]
The operators $P$ and $Q_{\mathcal O}$ are self-adjoint, and
$Q_{\mathcal O}^2=Q_{\mathcal O}$.  Hence
\[
\begin{split}
    F(P^t,\mathcal O)
    &=\operatorname{Tr}\left(
       (Q_{\mathcal O}P^t)^*Q_{\mathcal O}P^t
      \right)\\
    &=\operatorname{Tr}(P^tQ_{\mathcal O}P^t)
     =\operatorname{Tr}(Q_{\mathcal O}P^{2t}),
\end{split}
\]
which proves \eqref{eq:F-lag-trace}.  Summing this identity proves
\eqref{eq:F-wT-trace}.  Since $0\leq\gamma<1$ and
$\|P^2\|_{2\to2,\pi}\leq1$, the Neumann series gives
\[
    \sum_{t=1}^\infty\gamma^{t-1}P^{2t}
    =P^2\sum_{s=0}^\infty(\gamma P^2)^s
    =P^2(I-\gamma P^2)^{-1},
\]
which proves \eqref{eq:F-gamma-trace}.
\end{proof}

It is important to distinguish the objective in
\eqref{eq:F-lag-t} from the $t$-step convergence of a repeatedly
averaged chain.  In general,
\begin{equation*}
    G_{\mathcal O}P^t
    \neq (G_{\mathcal O}P)^t.
\end{equation*}
The left-hand side averages the initial state within its block and
then evolves the original chain $P$ for $t$ steps. An objective based on $(G_{\mathcal O}P)^t$ would instead measure the actual $t$-step convergence of the repeatedly averaged
kernel and would be nonlinear in the partition projection. The present formulation is attractive precisely because it retains the linear trace representation in Proposition~\ref{prop:multi-horizon-identities}.

The lagged distance is nonincreasing in $t$:
\begin{equation*}
    F(P^{t+1},\mathcal O)
    \leq F(P^t,\mathcal O),
    \qquad t\geq1.
\end{equation*}
This follows from the ideal property of the Hilbert--Schmidt norm
$\|\cdot\|_{F,\pi}$. Indeed, since $\Pi P=\Pi$ and
$\|P\|_{2\to2,\pi}\leq1$,
\[
\begin{aligned}
	F(P^{t+1},\mathcal O)^{1/2}
	&=\left\|
	\bigl(G_{\mathcal O}P^t-\Pi\bigr)P
	\right\|_{F,\pi}\\
	&\leq
	F(P^t,\mathcal O)^{1/2}\|P\|_{2\to2,\pi}\\
	&\leq F(P^t,\mathcal O)^{1/2}.
\end{aligned}
\] Thus the objectives in \eqref{eq:F-wT} and \eqref{eq:F-gamma-infinity} aggregate the distance to stationarity as it decays over time.

There is also a multi-horizon cut interpretation.  For either
$B=B_w^T$ or $B=B_\gamma^\infty$, let $c$ satisfy $B\1=c\1$.
Equation \eqref{eq:general-trace-cut} gives
\begin{equation*}
    \operatorname{Tr}(Q_{\mathcal O}B)
    =c(k-1)-\operatorname{Ncut}_B(\mathcal O).
\end{equation*}
If $\sum_{t=1}^T w_t=1$, then $B_w^T$ is a reversible Markov
kernel obtained by mixing the even-step kernels $P^{2t}$.  Likewise,
$(1-\gamma)B_\gamma^\infty$ is the reversible Markov kernel
\[
    (1-\gamma)\sum_{t=1}^\infty
       \gamma^{t-1}P^{2t}.
\]
Hence, after a positive rescaling, both objectives are
reverse normalized-cut objectives for mixtures of several time
horizons.

\subsection{Spectral filters and persistent modes}

For the simultaneous eigenbasis fixed in
Section~\ref{sec:preliminaries}, write
\[
P\phi_j=\xi_j\phi_j,
\qquad
\eta_j=\xi_j^2,
\]
where the $\eta_j$ remain ordered as in Section~\ref{sec:preliminaries}.
Thus
the signed eigenvalues $\xi_j$ are indexed according to the
nondecreasing order of their squares; they are not, in general,
ordered in the same way as the signed eigenvalues
$\lambda_1,\ldots,\lambda_{n-1}$ in
Section~\ref{sec:additive-mixture}.
The eigenvalues of $B_w^T$ and $B_\gamma^\infty$ on
$\ell^2_0(\pi)$ are, respectively,
\begin{equation}
	\beta_{w,j}^T
	:=\sum_{t=1}^T w_t\eta_j^t,
	\qquad
	\beta_{\gamma,j}^\infty
	:=\frac{\eta_j}{1-\gamma\eta_j}.
	\label{eq:filtered-eigenvalues}
\end{equation}

Under the standing assumptions $w_t\geq0$, with the weights not all
zero, and $0\leq\gamma<1$, both coefficient maps
$u\mapsto\sum_{t=1}^T w_tu^t$ and
$u\mapsto u/(1-\gamma u)$ are strictly increasing on $[0,1]$.
They therefore preserve the ordering, multiplicities, and eigenspaces
of $P^2$.

If $\eta_{k-1}<\eta_k$, all the multi-horizon objectives have the
same uniquely determined bottom $(k-1)$-dimensional spectral
subspace as the one-step objective $F$.  Their raw embeddings agree
up to an orthogonal change of basis.  If
$\eta_{k-1}=\eta_k$, a tied eigenspace straddles the spectral cutoff,
so the selected $(k-1)$-dimensional subspace need not be unique.  In
that case, we fix the same choice of subspace and eigenbasis for all
the objectives.  What changes across the objectives are the
eigenvalue penalties, the filtered eigengap and spectral width, and
the ranking of genuine partitions by their exact objective values.

Recall the projection masses $a_j(\mathcal O)$ from
\eqref{eq:spectral-projection-mass}.
Equations \eqref{eq:F-wT-trace}--\eqref{eq:filtered-eigenvalues}
give
\begin{align}
    F_w^T(\mathcal O)
    &=\sum_{j=1}^{n-1}
      \beta_{w,j}^T a_j(\mathcal O),
    \label{eq:F-w-mode-expansion}\\
    F_\gamma^\infty(\mathcal O)
    &=\sum_{j=1}^{n-1}
      \frac{\eta_j}{1-\gamma\eta_j}
      a_j(\mathcal O).
    \label{eq:F-gamma-mode-expansion}
\end{align}

\subsection{Why modes with \texorpdfstring{$|\xi|\approx1$}
{|xi| approximately 1} receive a large penalty}

Here a \emph{mode} means an eigenfunction $\phi_j$ of the Markov
operator $P$, satisfying $P\phi_j=\xi_j\phi_j$.  The eigenfunction
describes a pattern across the state space, while the eigenvalue
determines its temporal evolution: after $t$ steps, this pattern is
multiplied by $\xi_j^t$.  

The term \emph{spectral leakage} refers to overlap of the feasible
centered block-constant subspace
$\operatorname{Ran}(Q_{\mathcal O})$ with modes outside the relaxed
bottom spectral subspace.  The mass-balance identity
\eqref{eq:mass-balance} says exactly that
\[
    \delta(\mathcal O)
    =\sum_{j=k}^{n-1}a_j(\mathcal O).
\]
Thus $a_j(\mathcal O)$ for $j\geq k$ is precisely the projection mass
that has leaked from the relaxed bottom subspace into higher modes
during rounding.

Consider a single eigenmode satisfying $P\phi=\xi\phi$.  If its
overlap with $\operatorname{Ran}(Q_{\mathcal O})$ is $a$, then its
contribution at lag $t$ is
\begin{equation}
    a|\xi|^{2t}.
    \label{eq:single-mode-lag-cost}
\end{equation}
A mode with small $|\xi|$ is rapidly suppressed by $P^t$, so
\eqref{eq:single-mode-lag-cost} quickly becomes negligible.  A mode
with $|\xi|$ close to one persists for many steps.  Its total
discounted contribution is
\begin{equation}
    a\sum_{t=1}^\infty
       \gamma^{t-1}|\xi|^{2t}
    =a\frac{|\xi|^2}{1-\gamma|\xi|^2}
    =a\frac{|\xi|^2}
      {(1-\gamma)+\gamma(1-|\xi|^2)}.
    \label{eq:single-mode-discounted-cost}
\end{equation}
Relative to the one-step coefficient $|\xi|^2$, discounting
multiplies the weight by
\begin{equation}
    \frac1{1-\gamma|\xi|^2},
    \label{eq:multi-horizon-amplification}
\end{equation}
which is nondecreasing in $|\xi|$, and strictly increasing when $\gamma > 0$.  When both $1-\gamma$ and
$1-|\xi|^2$ are small, the denominator in
\eqref{eq:single-mode-discounted-cost} is small and the penalty is
large.  By contrast, for a small $|\xi|$, the multiplier in
\eqref{eq:multi-horizon-amplification} is close to one.  This is the
precise meaning of the statement that the discounted objective
strongly penalizes leakage into modes with $|\xi|\approx1$.
Negative eigenvalues close to $-1$ are penalized as well, because
they produce persistent period-two oscillations and their squared
magnitudes also decay slowly.

\subsection{The multi-horizon spectral sweep and its guarantee}

For the remainder of this subsection, let $F_\star$ denote either
$F_w^T$ or $F_\gamma^\infty$, and let
\[
    \beta_j^\star
    =\begin{cases}
       \beta_{w,j}^T,&F_\star=F_w^T,\\
       \beta_{\gamma,j}^\infty,
          &F_\star=F_\gamma^\infty.
     \end{cases}
\]
These eigenvalues are ordered increasingly.  Define
\begin{equation}
    L_\star:=\sum_{j=1}^{k-1}\beta_j^\star,
    \qquad
    g_\star:=\beta_{k}^\star-\beta_{k-1}^\star,
    \qquad
    W_\star:=\beta_{n-1}^\star-\beta_1^\star.
    \label{eq:multi-L-g-W}
\end{equation}

The multi-horizon spectral sweep is given in
Algorithm~\ref{alg:multi-horizon-sweep}.

\begin{algorithm}[H]
\caption{Multi-horizon spectral sweep}
\label{alg:multi-horizon-sweep}
\begin{enumerate}
    \item Compute the eigenfunctions $\phi_j$ for
    $j\in\llbracket k-1\rrbracket$ from the bottom
    spectral subspace of $P^2$, and form the raw embedding
    $\Phi$ in \eqref{eq:spectral-embedding}.
    \item \emph{Spectral rounding.} Approximately minimize the spectral rounding error
    $\delta(\mathcal O)$ by running weighted $k$-means on the embedded
    points $\Phi(x)$, with weights $\pi(x)$, from several
    initializations.  In view of Proposition~\ref{prop:kmeans-identity},
    the weighted $k$-means distortion is exactly $\delta(\mathcal O)$.
    Retain the terminal partitions having $k$ nonempty blocks and denote
    them by
    $\mathcal O^{(1)},\ldots,\mathcal O^{(\ell)}$.
    \item \emph{Candidate rescoring.} Evaluate the chosen original objective $F_\star$ on every
    candidate and return
    \[
        \widehat{\mathcal O}_\star
        \in\operatorname*{argmin}_{s\in\llbracket\ell\rrbracket}
          F_\star(\mathcal O^{(s)}).
    \]
\end{enumerate}
\end{algorithm}

For the discounted objective, we write
$\widehat{\mathcal O}_\gamma:=\widehat{\mathcal O}_\star$ when
$F_\star=F_\gamma^\infty$, matching the notation used in the
figures.

For $k=2$, the first two steps may instead use the exact scalar
threshold sweep of Section~\ref{sec:two-block}.  Although the
spectral embedding is the same as for the one-step objective, the
last step is essential: different time weights can select different
partitions from the candidate family.

Analogous to Theorem~\ref{thm:k-sandwich}, we give a multi-horizon version of spectral sandwich:

\begin{theorem}[Multi-horizon spectral sandwich]
\label{thm:multi-horizon-sandwich}
Let $k\in\llbracket 2,n-1\rrbracket$, and recall that $L_\star, g_\star, W_\star$ are introduced in \eqref{eq:multi-L-g-W}. For every $k$-block partition $\mathcal O$,
\begin{equation}
    L_\star+g_\star\delta(\mathcal O)
    \leq F_\star(\mathcal O)
    \leq L_\star+W_\star\delta(\mathcal O).
    \label{eq:multi-horizon-sandwich}
\end{equation}
\end{theorem}

\begin{proof}
By \eqref{eq:mass-balance},
\[
    \sum_{j=1}^{k-1}(1-a_j(\mathcal O))
    =\sum_{j=k}^{n-1}a_j(\mathcal O)
    =\delta(\mathcal O).
\]
Using the appropriate expansion in
\eqref{eq:F-w-mode-expansion} or
\eqref{eq:F-gamma-mode-expansion},
\[
\begin{split}
    F_\star(\mathcal O)-L_\star
    &={}
      \sum_{j=k}^{n-1}\beta_j^\star a_j(\mathcal O)
      -\sum_{j=1}^{k-1}
       \beta_j^\star(1-a_j(\mathcal O)).
\end{split}
\]
For the lower bound, use
$\beta_j^\star\geq\beta_{k}^\star$ in the first
sum and $\beta_j^\star\leq\beta_{k-1}^\star$ in the second.
For the upper bound, use
$\beta_j^\star\leq\beta_{n-1}^\star$ in the first sum and
$\beta_j^\star\geq\beta_1^\star$ in the second.  This gives
\eqref{eq:multi-horizon-sandwich}.
\end{proof}

Let
\[
    F_{\star,k}^*
    :=\min_{|\mathcal O|=k}F_\star(\mathcal O).
\]
We retain $\delta_k^*$ from Section~\ref{sec:k-block}.

Analogous to Corollary~\ref{cor:k-algorithm}, we state the guarantee of Algorithm~\ref{alg:multi-horizon-sweep}:

\begin{corollary}[Algorithmic guarantee for the multi-horizon sweep]
\label{cor:multi-horizon-algorithm}
Let $k\in\llbracket 2,n-1\rrbracket$.
The output of Algorithm~\ref{alg:multi-horizon-sweep} satisfies
\begin{equation}
    0\leq
    F_\star(\widehat{\mathcal O}_\star)-F_{\star,k}^*
    \leq
    W_\star
    \min_{s\in\llbracket\ell\rrbracket}
      \delta(\mathcal O^{(s)})
    -g_\star\delta_k^*.
    \label{eq:multi-horizon-algorithm-bound}
\end{equation}
If at least one candidate is a $\rho$-approximation for weighted
$k$-means, so that
\[
    \min_{s\in\llbracket\ell\rrbracket}
      \delta(\mathcal O^{(s)})
    \leq\rho\delta_k^*,
\]
then
\begin{equation}
    F_\star(\widehat{\mathcal O}_\star)-F_{\star,k}^*
    \leq(\rho W_\star-g_\star)\delta_k^*.
    \label{eq:multi-horizon-rho-bound}
\end{equation}
\end{corollary}

\begin{proof}
Minimizing the lower bound in
Theorem~\ref{thm:multi-horizon-sandwich} over all $k$-block
partitions gives
\[
    F_{\star,k}^*\geq L_\star+g_\star\delta_k^*.
\]
Choose a generated candidate attaining
$\min_{s\in\llbracket\ell\rrbracket}
\delta(\mathcal O^{(s)})$.  Since
$\widehat{\mathcal O}_\star$ has the smallest exact
$F_\star$ among all generated candidates, the upper bound in
Theorem~\ref{thm:multi-horizon-sandwich} gives
\[
    F_\star(\widehat{\mathcal O}_\star)
    \leq
    L_\star+
    W_\star\min_{s\in\llbracket\ell\rrbracket}
      \delta(\mathcal O^{(s)}).
\]
Subtracting proves \eqref{eq:multi-horizon-algorithm-bound}, and the
$\rho$-approximation assumption proves
\eqref{eq:multi-horizon-rho-bound}.
\end{proof}

\subsection{Choosing the horizon weights}

Several choices have simple interpretations.

\begin{itemize}
    \item Setting $T=1$ and $w_1=1$ recovers the original objective
    $F(\mathcal O)$.

    \item Uniform weights $w_t=1$ give the spectral penalty
    \[
        u+\cdots+u^T
        =\frac{u(1-u^T)}{1-u},
    \]
    whose value at $u=1$ is understood to be $T$.  This treats the
    first $T$ lags equally.

    \item For the infinite-horizon objective $F_\gamma^\infty$ with $\gamma \in [0,1)$, the weights
    $w_t=\gamma^{t-1}$ decay geometrically.  After multiplication by
    $1-\gamma$, they become the probability mass function
    \[
    \mathbb P(\tau=t)
    =(1-\gamma)\gamma^{t-1},
    \qquad t=1,2,\ldots,
    \]
    of a geometric random horizon $\tau$. It follows that
    \begin{equation*}
    	(1-\gamma)F_\gamma^\infty(\mathcal O)
    	=\mathbb E\bigl[F(P^\tau,\mathcal O)\bigr].
    \end{equation*}
    Its mean, or effective horizon, is $(1-\gamma)^{-1}$.
\end{itemize}

Positive multi-horizon weights do not create a new relaxed spectral
subspace: they apply an increasing filter to the same eigenvalues of
$P^2$.  Their value is instead to alter how severely different modes
are penalized and to choose the final genuine partition according to
performance over the desired range of time scales.  

\section{Numerical experiments}
\label{sec:numerical-experiments}

All three multiway procedures use a spectral embedding followed by
weighted $k$-means, paralleling the standard spectral-clustering
rounding paradigm \citep{NgJordanWeiss2001}.  In the present setting,
however, weighted $k$-means is used only to generate partitions with
small spectral rounding error; the resulting candidates are ranked
using the corresponding original Markov-chain objective.
Table~\ref{tab:spectral-sweep-comparison} summarizes the three
procedures.

\begin{table}[H]
	\centering
	\small
	\caption{Comparison of the three multiway spectral sweep procedures.
		The second column gives the spectral rounding surrogate approximately
		minimized by weighted $k$-means, the third gives the original
		candidate-selection objective, and the rightmost column records the
		notation for the selected output.}
	\label{tab:spectral-sweep-comparison}
	
	\begingroup
	\renewcommand{\arraystretch}{1.25}
	\begin{tabularx}{\textwidth}{
			@{}
			>{\raggedright\arraybackslash}p{0.23\textwidth}
			>{\raggedright\arraybackslash}p{0.22\textwidth}
			>{\raggedright\arraybackslash}X
			>{\centering\arraybackslash}p{0.13\textwidth}
			@{}
		}
		\toprule
		Spectral sweep and embedding
		&
		Weighted $k$-means surrogate
		&
		Exact candidate-selection objective
		&
		Selected output
		\\
		\midrule
		\addlinespace
		
		One-step Frobenius sweep,
		Algorithm~\ref{alg:multiway-sweep}.
		
		Bottom spectral subspace of $P^2$ and embedding $\Phi$.
		&
		$\delta(\mathcal O)$, whose weighted $k$-means representation is
		given in Proposition~\ref{prop:kmeans-identity}.
		&
		$
		F(\mathcal O)
		=
		\|G_{\mathcal O}P-\Pi\|_{F,\pi}^2
		$
		(see \eqref{eq:F-definition}).
		&
		$\widehat{\mathcal O}_F$ 
		\\
		\addlinespace
		
		Additive-mixture sweep,
		Algorithm~\ref{alg:additive-mixture-sweep}.
		
		Bottom spectral subspace of $P$ and embedding $\Psi$.
		&
		$\delta_P(\mathcal O)$, whose weighted $k$-means representation is
		given in \eqref{eq:delta-P-identities}.
		&
		$
		H(\mathcal O)
		=
		\|A_{\mathcal O}-\Pi\|_{F,\pi}^2
		$
		(see \eqref{eq:H-definition}).
		&
		$\widehat{\mathcal O}_H$
		\\
		\addlinespace
		
		Multi-horizon sweep,
		Algorithm~\ref{alg:multi-horizon-sweep}.
		
		Bottom spectral subspace of $P^2$ and the same embedding $\Phi$ as in the
		one-step sweep.
		&
		The same rounding surrogate $\delta(\mathcal O)$ as in
		Proposition~\ref{prop:kmeans-identity}.
		&
		$
		F_\star(\mathcal O)
		\in
		\bigl\{
		F_w^T(\mathcal O),
		F_\gamma^\infty(\mathcal O)
		\bigr\}
		$
		(see \eqref{eq:F-wT} and
		\eqref{eq:F-gamma-infinity}).
		&
		$\widehat{\mathcal O}_\star$

		($\widehat{\mathcal O}_\gamma$ when
		$F_\star=F_\gamma^\infty$)
		\\
		\bottomrule
	\end{tabularx}
	\endgroup
\end{table}

As Table~\ref{tab:spectral-sweep-comparison} shows, the one-step
Frobenius and multi-horizon objectives share the same embedding
$\Phi$ and the same rounding surrogate $\delta$, because their
relaxed operators are increasing spectral functions of $P^2$.
The additive-mixture objective instead uses the bottom spectral subspace of
$P$ and therefore has the distinct embedding $\Psi$ and rounding
surrogate $\delta_P$.  In every case, the final candidate is selected
using the original objective rather than the rounding surrogate.

\subsection{Implementation and evaluation protocol}
\label{subsec:numerical-protocol}

The Python code used to run all three numerical experiments and
generate the associated figures is available at
\url{https://github.com/mchchoi/spectral-partitioning}.  The
implementations construct the finite-state transition matrices and
stationary distributions explicitly.  Weighted $k$-means uses
weighted $k$-means++ initialization
\citep{ArthurVassilvitskii2007}, followed by weighted Lloyd
iterations.  Unless stated otherwise, each multiway experiment uses
$80$ starts.  Every distinct nonempty terminal partition is retained,
and the winner is the partition with the smallest value of the
relevant original objective among those candidates.  Thus
``optimized'' below always means optimized over the generated
candidate pool; it does not mean that the globally optimal discrete
partition has been found.  Algorithm~\ref{alg:two-block-sweep}
instead evaluates every nontrivial threshold cut of the bottom
nonconstant eigenfunction of $P^2$.

We use two blocks for Algorithm~\ref{alg:two-block-sweep} and $k=4$ blocks
for Algorithms~\ref{alg:multiway-sweep},
\ref{alg:additive-mixture-sweep}, and
\ref{alg:multi-horizon-sweep}.  The multi-horizon experiment uses
$\gamma=0.9$, corresponding to mean geometric horizon $10$.  To keep
the figures compact, their label
$F_\gamma$ means $F_\gamma^\infty$.  

In addition to the objective for which each procedure was designed, we
report
\begin{equation}
    d_{\mathrm{TV}}^{\mathrm{wc}}(K,t)
    :=\max_{x\in\X}
      \|K^t(x,\cdot)-\pi\|_{\mathrm{TV}}.
    \label{eq:worst-case-TV}
\end{equation}
This is a useful diagnostic for the repeatedly applied output kernel,
but it is not the same quantity as the multi-horizon objective.  In
particular,
\[
    (G_{\mathcal O}P)^t
    \ne G_{\mathcal O}P^t
\]
in general.  Consequently, a smaller $F_\gamma^\infty$ need not imply
a smaller repeated-kernel total-variation distance.

\subsection{Controlled-spectrum weighted-graph experiment}
\label{subsec:controlled-spectrum-experiment}

\paragraph{Construction.}
Let
\[
    \X=\{0,1\}\times\mathbb Z_{20}.
\]
Within each copy of $\mathbb Z_{20}$, consecutive vertices are joined
by edges of weight one.  A single bridge of weight $\varepsilon$ joins
$(0,0)$ to $(1,1)$.  Thus, before lazification, the construction is a
weighted dumbbell graph: two $20$-cycles are
connected by a single edge, which becomes a pronounced bottleneck when
$\varepsilon$ is small.  If $W_\varepsilon$ is the resulting symmetric
weight matrix and $D_\varepsilon$ its degree matrix, define
\begin{equation*}
    P_{\alpha,\varepsilon}
    =\alpha I+(1-\alpha)D_\varepsilon^{-1}W_\varepsilon,
    \qquad
    \pi(x)=
    \frac{D_\varepsilon(x,x)}
         {\operatorname{Tr}(D_\varepsilon)}.
\end{equation*}
The bridge preserves bipartiteness before lazification.  Hence the
unlazified random walk has eigenvalue $-1$.  Since lazification maps
each eigenvalue $\mu$ to
\[
\alpha+(1-\alpha)\mu,
\]
the eigenvalue ordering introduced in Section~\ref{sec:additive-mixture} gives
\begin{equation*}
	\lambda_1(P_{\alpha,\varepsilon})=2\alpha-1.
\end{equation*}
Meanwhile, a small bridge weight produces a largest non-unit
eigenvalue $\lambda_{n-1}(P_{\alpha,\varepsilon})$ close to one.
The two parameters therefore let us vary a large negative signed mode
and a slowly decaying positive mode in a controlled way.

\paragraph{A signed-spectrum instance.}
We first take $\alpha=0.08$ and $\varepsilon=0.03$.  Here
\[
\lambda_1(P_{0.08,0.03})=-0.84,
\qquad
\lambda_{n-1}(P_{0.08,0.03})=0.998748.
\]
Thus $\lambda_1$ is the smallest eigenvalue, while
$\lambda_{n-1}$ is the largest non-unit eigenvalue, equivalently the
second-largest eigenvalue of the full kernel after the unit
eigenvalue.

The spectral relaxation for the additive-mixture objective uses the
algebraically smallest eigenvalues of $P$.  In contrast, the
relaxation based on $P^2$ uses the smallest squared eigenvalues
$\lambda_j^2$, equivalently the eigenvalues of $P$ having the smallest
absolute values.  Consequently, the relaxed bottom spectral subspaces used
by $P$ and $P^2$ are markedly different in this example.
Figure~\ref{fig:controlled-partitions} displays all four outputs.
Algorithms~\ref{alg:multiway-sweep} and
\ref{alg:multi-horizon-sweep} select the same four-block partition in
this instance, whereas the additive-mixture output disagrees with it.
This is the intended signed-spectrum distinction: $F$ treats large
positive and negative eigenvalues symmetrically through $P^2$, while
$H$ retains their signs through $P$.

\begin{figure}[tbp]
    \centering
    \includegraphics[width=0.96\textwidth]
        {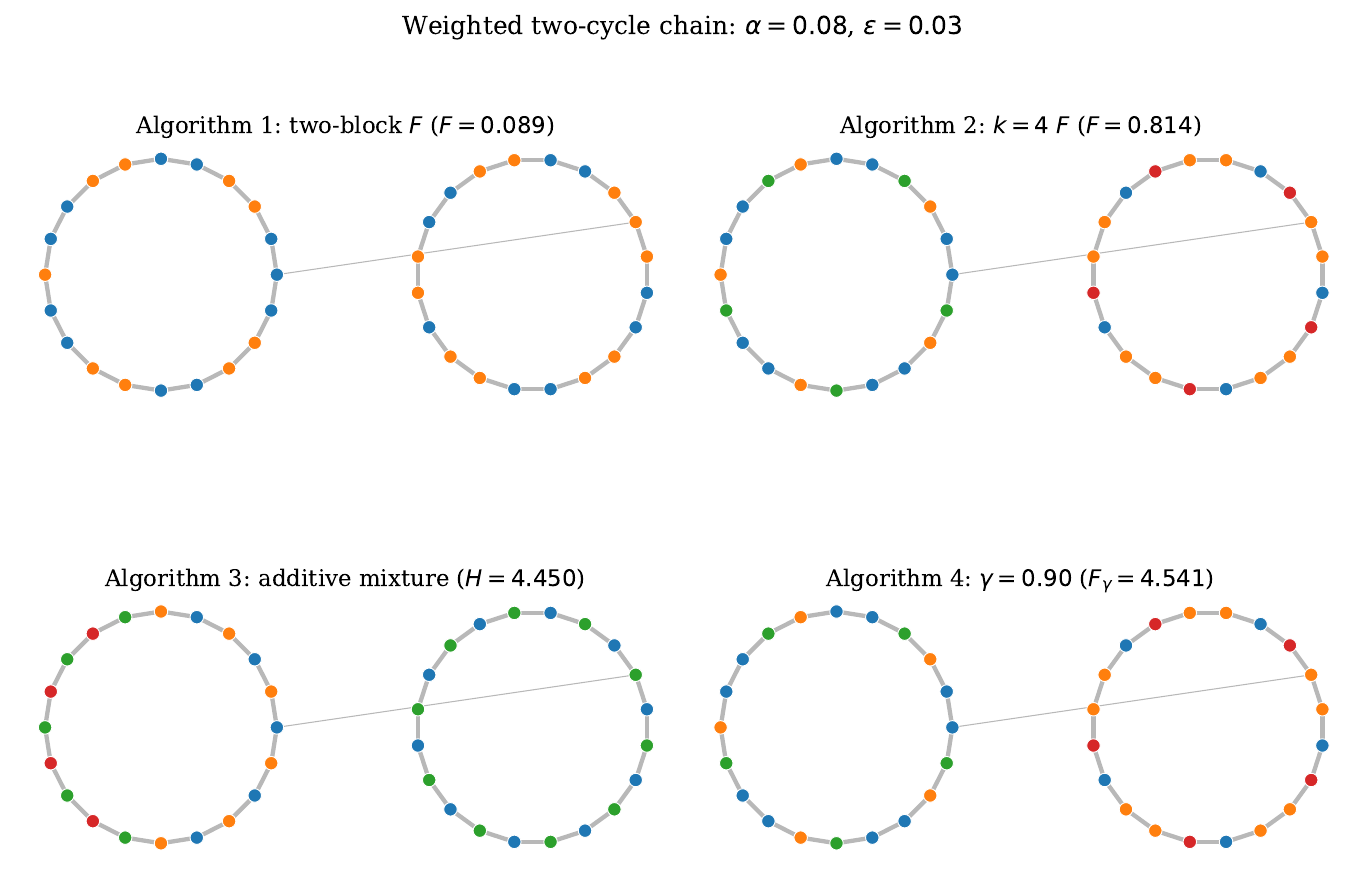}
    \caption{Outputs of the four procedures for the controlled-spectrum
    chain with $\alpha=0.08$ and $\varepsilon=0.03$.  Algorithm~\ref{alg:two-block-sweep} uses
    two blocks; Algorithms~\ref{alg:multiway-sweep}--\ref{alg:multi-horizon-sweep} use $k=4$ blocks.  Raw objective values should
    therefore not be used to rank procedures having different numbers
    of blocks.}
    \label{fig:controlled-partitions}
\end{figure}

The left panel of Figure~\ref{fig:controlled-diagnostics} shows the
different relaxed eigenspaces.  The center panel evaluates the common
lagged Frobenius profile $F(P^t,\mathcal O)$, and the right panel
separately evaluates repeated-kernel convergence.  At repeated-kernel
time $10$, the exact worst-case total-variation distances are
\begin{center}
\begin{tabular}{@{}lccccc@{}}
    \toprule
    kernel $K$
    &$P$
    &$G_{\widehat{\mathcal O}_2}P$
    &$G_{\widehat{\mathcal O}_F}P$
    &$A_{\widehat{\mathcal O}_H}$
    &$G_{\widehat{\mathcal O}_\gamma}P$\\
    \midrule
    $d_{\mathrm{TV}}^{\mathrm{wc}}(K,10)$
    &$0.652$
    &$1.01\times10^{-11}$
    &$6.55\times10^{-7}$
    &$0.129$
    &$6.55\times10^{-7}$\\
    \bottomrule
\end{tabular}
\end{center}
All four constructions accelerate this small chain in this diagnostic,
although an iteration of the additive mixture has a different
computational interpretation from an iteration of
$G_{\mathcal O}P$.

\begin{figure}[tbp]
    \centering
    \includegraphics[width=\textwidth]
        {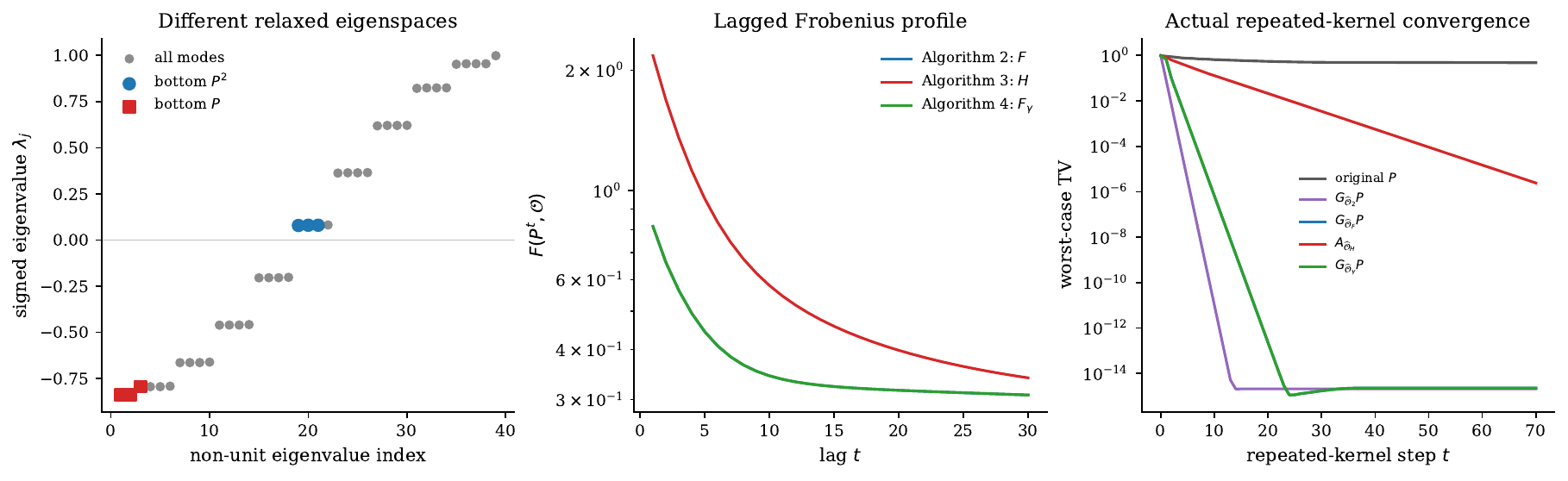}
    \caption{Diagnostics for the controlled-spectrum chain with
    $\alpha=0.08$ and $\varepsilon=0.03$.  Left: all non-unit eigenvalues $\lambda_j(P)$ of $P$ (grey). Blue circles mark the $k-1 = 3$ eigenvalues selected by the bottom spectrum
    of $P^2$, while red squares mark the $k-1 = 3$ algebraically smallest
    eigenvalues of $P$.  Center: the lagged Frobenius objective
    $F(P^t,\mathcal O)$.  Right: the worst-case total-variation
    distance $d_{\mathrm{TV}}^{\mathrm{wc}}(K,t)$ from
    \eqref{eq:worst-case-TV} for each displayed kernel $K$.}
    \label{fig:controlled-diagnostics}
\end{figure}

\paragraph{Empirical algorithmic phase diagram.}
We next use the grid
\[
    \alpha\in
    \{0.02,0.08,0.16,0.25,0.34,0.42,0.48,0.50\}, \, \varepsilon\in
    \{0.001,0.006,0.03,0.05,0.15,0.75\}.
\]
Each phase-diagram cell is a separately seeded $40$-start run and
need not reproduce the $80$-start representative experiment at the
same parameter values.  These are restart-dependent algorithmic
observations, not global discrete optima or phase boundaries.  For two four-block partitions, define their weighted
disagreement by
\begin{equation}
    d_\pi(\mathcal O,\mathcal O')
    =1-\max_{\sigma\in\mathfrak S_4}
       \sum_{i=1}^4\pi(O_i\cap O'_{\sigma(i)}).
    \label{eq:weighted-partition-disagreement}
\end{equation}
Here $\mathfrak S_4$ denotes the symmetric group of all $4!$
permutations of $\llbracket 4\rrbracket$.  The maximization removes the arbitrary
labelling of the blocks by optimally matching the blocks of
$\mathcal O$ with those of $\mathcal O'$.
Figure~\ref{fig:controlled-phase} shows that the selected partitions
can change substantially with both parameters.  Across the $48$ grid
points, the largest observed $F$--$H$ disagreement is $0.725$ and the
largest $F$--$F_\gamma^\infty$ disagreement is $0.550$.  

\begin{figure}[tbp]
    \centering
    \includegraphics[width=0.94\textwidth]
        {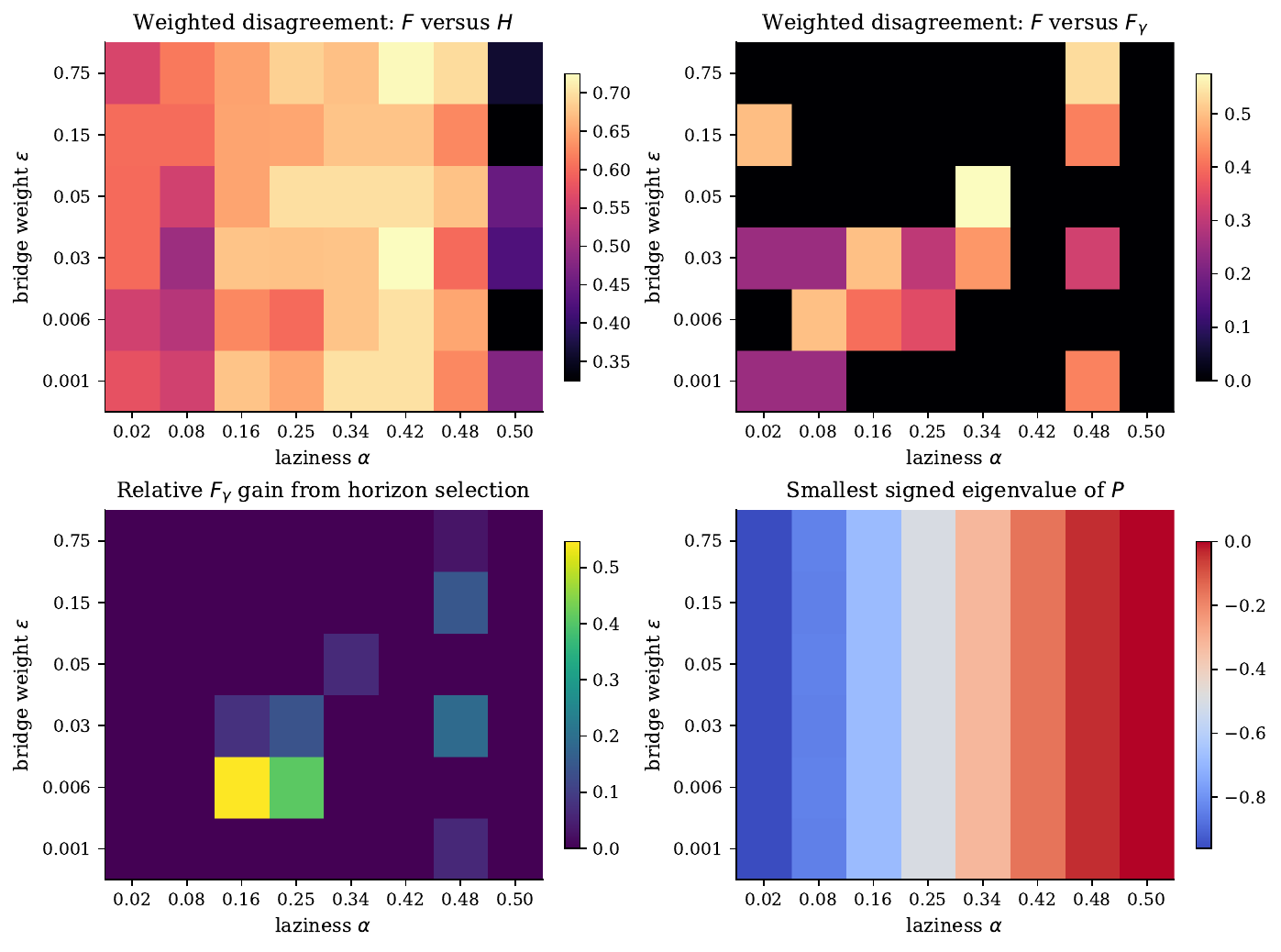}
    \caption{Empirical controlled-spectrum phase diagram.  Weighted
    disagreement is defined in
    \eqref{eq:weighted-partition-disagreement}. For each grid point, let $\widehat{\mathcal O}_F$ and
    $\widehat{\mathcal O}_\gamma$ denote the candidates selected from the
    same bottom-$P^2$ candidate pool by minimizing $F$ and
    $F_\gamma^\infty$, respectively.  The lower-left panel displays
    $
    1-
    \frac{F_\gamma^\infty(\widehat{\mathcal O}_\gamma)}
    {F_\gamma^\infty(\widehat{\mathcal O}_F)},
    $
    the relative reduction in the discounted multi-horizon objective. }
    \label{fig:controlled-phase}
\end{figure}

The largest horizon gain on this grid occurs at
$\alpha=0.16$ and $\varepsilon=0.006$, see
Figure~\ref{fig:controlled-horizon}.  The two selected partitions
disagree on stationary mass $0.550$.  The multi-horizon consideration increases the lag-one Frobenius objective only slightly,
\[
F(P,\widehat{\mathcal O}_F)=0.5805
\quad\text{to}\quad
F(P,\widehat{\mathcal O}_\gamma)=0.5903,
\]
but decreases the discounted multi-horizon objective from
\[
F_\gamma^\infty(\widehat{\mathcal O}_F)=3.4202
\quad\text{to}\quad
F_\gamma^\infty(\widehat{\mathcal O}_\gamma)=1.5517.
\]
At lag $10$, the corresponding values of $F(P^{10},\mathcal O)$ are
$0.3002$ and $0.05818$.  In this particular instance, the separate
repeated-kernel diagnostic also improves at time $10$, from
$1.10\times10^{-6}$ to $3.61\times10^{-8}$.  The latter agreement is
empirical and is not implied by the multi-horizon objective.

This example shows that longer-horizon rescoring can materially alter
the selected partition even when the spectral candidate pool is
unchanged.  We next test the procedures on an Ising model.

\begin{figure}[tbp]
    \centering
    \includegraphics[width=0.94\textwidth]
        {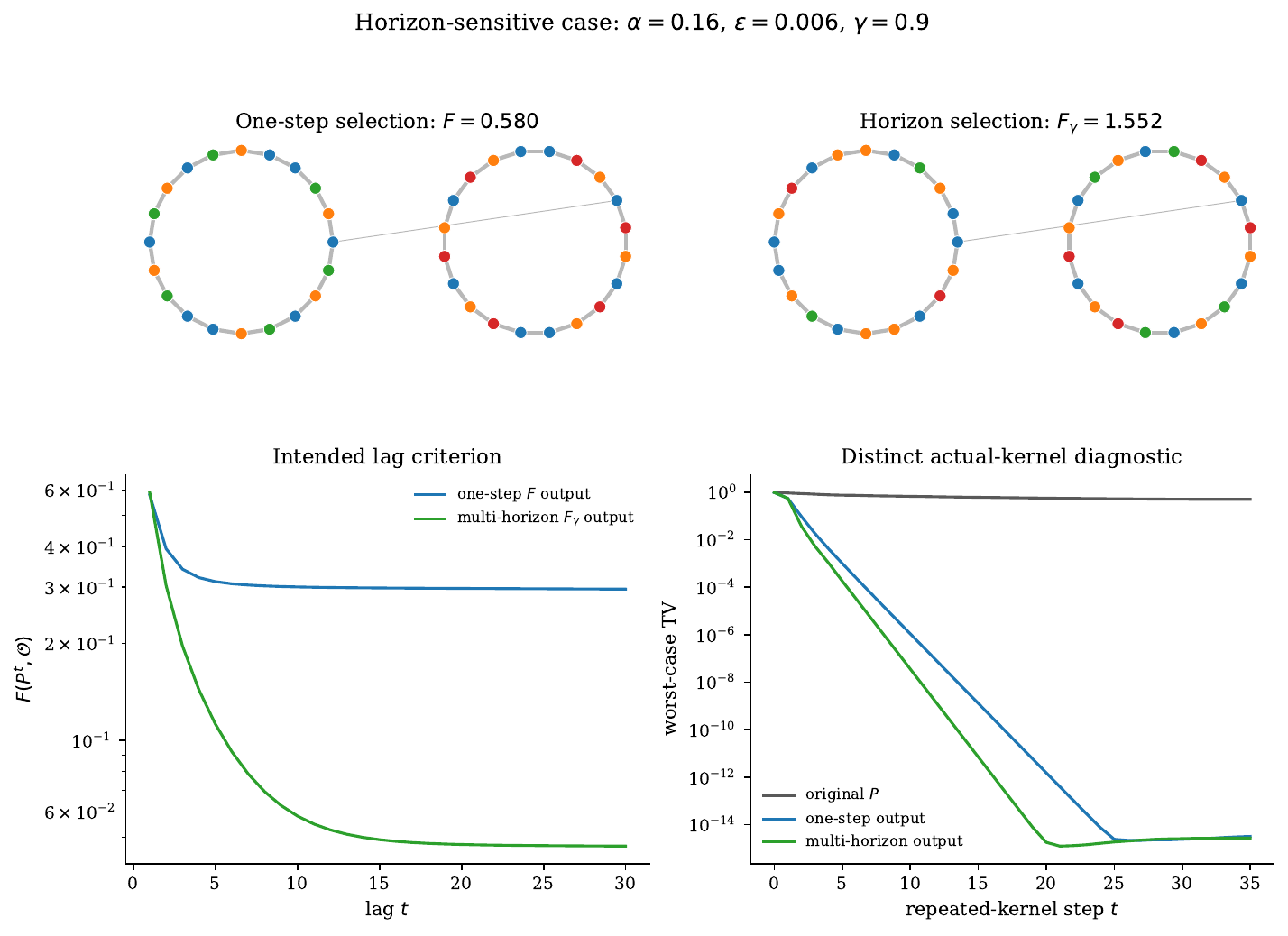}
    \caption{The most horizon-sensitive point $(\alpha,\varepsilon)=(0.16,0.006)$ on the tested grid. The lower-left panel evaluates the lagged Frobenius profile
    $F(P^t,\mathcal O)$; the lower-right panel reports the
    distinct diagnostic $d_{\mathrm{TV}}^{\mathrm{wc}}(K,t)$ from
    \eqref{eq:worst-case-TV} for each displayed kernel $K$.}
    \label{fig:controlled-horizon}
\end{figure}

\subsection{Exact mean-field Ising experiment}
\label{subsec:ising-experiment}

\paragraph{Model and kernel.}
Throughout this subsection, we consider $\mathcal{X} = \{-1,+1\}^d$. For $x\in\{-1,+1\}^d$, define the magnetization of $x$ to be
\[
    M(x):=\sum_{i=1}^d x_i
\]
and the Gibbs distribution $\pi_{\beta,h}$ at inverse temperature $\beta \geq 0$ and external field $h \in \mathbb{R}$ to be
\begin{equation*}
    \pi_{\beta,h}(x)
    :=\frac{1}{Z_{\beta,h}}
      \exp\left\{\frac{\beta}{2d}M(x)^2+hM(x)\right\},
\end{equation*}
where $Z_{\beta,h}$ is the normalization constant. The transition kernel is exact random-site heat-bath dynamics: a site
$i$ is selected uniformly and is resampled according to
\begin{equation*}
    \mathbb P(X_i'=+1\mid x_{-i})
    =\left[
      1+\exp\left\{
        -\frac{2\beta}{d}\sum_{j\ne i}x_j-2h
      \right\}
    \right]^{-1}.
\end{equation*}
We take $d = 8$, enumerate all $2^8=256$ states and consider
\[
    \beta\in\{0.4,0.8,1.0,1.2,1.5\},
    \qquad
    h\in\{0,0.25\}.
\]
At zero field, this grid straddles the Ising critical inverse
temperature $\beta_c=1$; the corresponding zero-field asymptotic
mixing-time transition is studied in \citep{DingLubetzkyPeres2009}.  Our experiment is instead an exact finite-$d$ benchmark and makes no asymptotic claim.

The baselines used in the Ising experiment are summarized
in Table~\ref{tab:ising-baselines}.

\begin{table}[H]
	\centering
	\small
	\caption{Baselines used in the Ising experiment with $k=4$.}
	\label{tab:ising-baselines}
	\begingroup
	\renewcommand{\arraystretch}{1.18}
	\begin{tabularx}{\textwidth}{
			@{}
			>{\raggedright\arraybackslash}p{0.20\textwidth}
			>{\raggedright\arraybackslash}X
			>{\raggedright\arraybackslash}p{0.27\textwidth}
			@{}
		}
		\toprule
		Baseline & Construction & Evaluation \\
		\midrule
		
		Magnetization bins
		&
		Order the magnetization levels $M(x)$ and group whole levels into four
		nonempty contiguous bins.  No magnetization level is split.
		&
		One deterministic partition.
		\\
		
		Coordinate blocks
		&
		The four blocks are determined by the four possible values of
		$(x_1,x_2)\in\{-1,+1\}^2$.
		&
		One deterministic partition.
		\\
		
		Random nonempty assignments
		&
		Randomly permute the states, assign one state to each of the
		four labeled blocks, and assign every remaining state
		independently and uniformly to a block.
		&
		Mean and one empirical standard deviation over $80$ draws in
		Figure~\ref{fig:ising-objectives}.  This is not uniform over
		unlabeled set partitions.
		\\
		
		\bottomrule
	\end{tabularx}
	\endgroup
\end{table}

\begin{figure}[tbp]
    \centering
    \includegraphics[width=\textwidth]{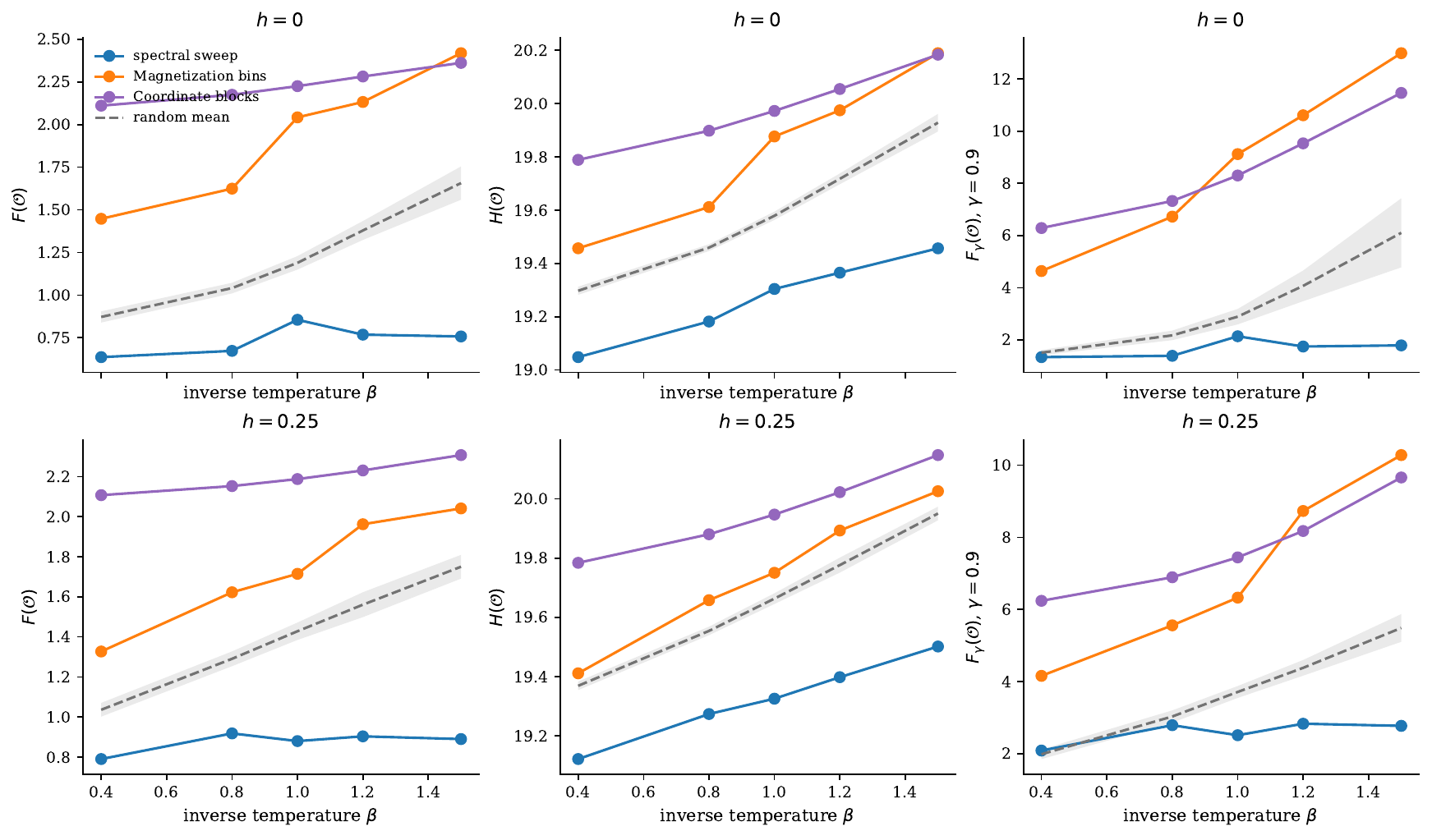}
    \caption{Exact Ising objective values for $d=8$ and $k=4$.
    The shaded band is one empirical standard deviation across $80$
    random nonempty label assignments.  Each blue curve evaluates the
    objective used to select that spectral-sweep output. Within each panel, lower values indicate better performance under the displayed criterion.}
    \label{fig:ising-objectives}
\end{figure}

\paragraph{Objective values.}
At every tested $(\beta,h)$ pair, the sweep selected for a given
objective has a smaller value of that objective than the two
deterministic baselines and the empirical mean of the random baseline.
At the low-temperature, zero-field point $(\beta,h)=(1.5,0)$, the exact
values are
\begin{center}
\begin{tabular}{@{}lccc@{}}
    \toprule
    objective
    &spectral-sweep output
    &magnetization bins
    &coordinate blocks\\
    \midrule
    $F$&$0.7454$&$2.4185$&$2.3616$\\
    $H$&$19.4566$&$20.1889$&$20.1833$\\
    $F_\gamma^\infty$&$1.7356$&$12.9934$&$11.4684$\\
    \bottomrule
\end{tabular}
\end{center}
Relative to the magnetization baseline, the reductions are $69.2\%$
for $F$, $3.63\%$ for $H$, and $86.6\%$ for
$F_\gamma^\infty$.  The percentage for $H$ must be read with care:
$H$ includes the large partition-independent constant $C_{P,k}$ in
\eqref{eq:H-constant}, so the raw percentage understates the relative
change of its partition-dependent term.

Heat-bath kernels have no negative eigenvalues
\citep{DyerGreenhillUllrich2014}. Numerically, in the representative
instance,
\[
\lambda_1=-6.5\times10^{-17}\approx0,
\qquad
\lambda_{n-1}=0.990077.
\]
The tiny negative value of $\lambda_1$ is attributable to
floating-point roundoff.  Consequently, in exact arithmetic, $P$ and
$P^2$ have the same eigenvalue ordering, up to ties, because all
eigenvalues of $P$ are nonnegative and the map $u\mapsto u^2$ is
increasing on $[0,1]$.

For $d=8$ and $k=4$, however, the cutoff after $k-1=3$
eigenfunctions lies inside a sevenfold eigenvalue.  Because the cutoff
lies inside a tied eigenspace, identical spectral ordering alone does
not determine a common embedding.  In the experiment, we deliberately
use the same selected three-dimensional subspace and the same
$\pi$-orthonormal basis for $P$, $P^2$, and the filtered objectives.
The three procedures therefore score a common candidate pool,
allowing their objective-dependent selection rules to be compared
without a confounding change of embedding.  A different choice inside
the tied eigenspace could produce a different pool and different
rounded outputs.

At $(\beta,h)=(1.5,0)$, the $F$ and
$F_\gamma^\infty$ outputs coincide, while the $H$ output differs on
only $0.015$ of the stationary mass.

\paragraph{Structure of the partitions.}
Figure~\ref{fig:ising-partitions} shows how the learned blocks intersect
magnetization levels.  The output is not an ordered magnetization
clustering: it puts most levels in one large block and uses the other
blocks mainly to refine central configurations.  At
$(\beta,h)=(1.5,0)$, the smallest block masses are $0.0119$, $0.0131$,
and $0.0119$ for the $F$, $H$, and multi-horizon outputs, respectively,
and the largest block has stationary mass about $0.96$.  This is
consistent with the reverse-clustering nature of the objectives, which
seek averaging blocks that suppress persistent spectral modes rather than
communities that preserve them.  It also indicates that a minimum-mass
constraint or a block-balance regularizer may be important in
applications.

\begin{figure}[tbp]
    \centering
    \includegraphics[width=\textwidth]{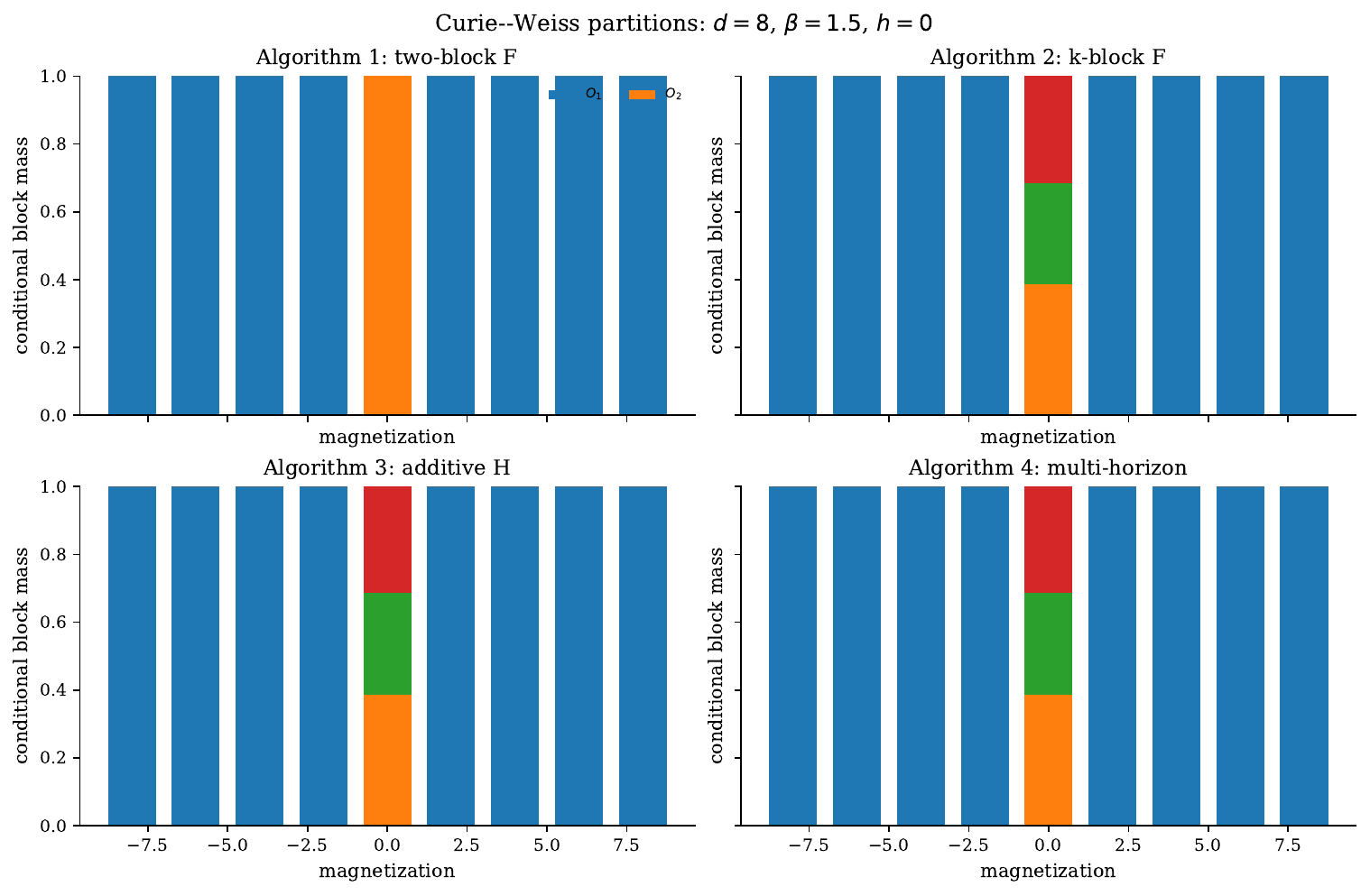}
    \caption{Conditional block allocation by magnetization for the
    exact Curie--Weiss chain at $(\beta,h)=(1.5,0)$.  Each bar gives the
    stationary conditional distribution of block labels within one
    magnetization level.}
    \label{fig:ising-partitions}
\end{figure}

\paragraph{Lagged Frobenius profile and actual kernel convergence.}
We denote by $\widehat{\mathcal O}_{\mathrm{mag}}$ the deterministic
four-block magnetization-bin partition described in
Table~\ref{tab:ising-baselines}. At lag $10$, the lagged Frobenius values are
\[
\begin{split}
    F(P^{10},\widehat{\mathcal O}_F)&=0.04268,\\
    F(P^{10},\widehat{\mathcal O}_H)&=0.04270,\\
    F(P^{10},\widehat{\mathcal O}_\gamma)&=0.04268,
\end{split}
\qquad
    F(P^{10},\widehat{\mathcal O}_{\mathrm{mag}})=0.97483.
\]
The corresponding repeated-kernel total-variation distances at time
$10$ are
\begin{center}
\begin{tabular}{@{}lccccc@{}}
    \toprule
    kernel $K$
    &$P$
    &$G_{\widehat{\mathcal O}_2}P$
    &$G_{\widehat{\mathcal O}_F}P$
    &$A_{\widehat{\mathcal O}_H}$
    &$G_{\widehat{\mathcal O}_\gamma}P$\\
    \midrule
    $d_{\mathrm{TV}}^{\mathrm{wc}}(K,10)$
    &$0.5277$
    &$1.53\times10^{-4}$
    &$2.46\times10^{-4}$
    &$0.0820$
    &$2.46\times10^{-4}$\\
    \bottomrule
\end{tabular}
\end{center}
As Figure~\ref{fig:ising-diagnostics} shows, the additive mixture
improves markedly on $P$ but converges more slowly per iteration than
the $G_{\mathcal O}P$ outputs in this example.  This does not conflict
with its objective: $H$ measures the one-step Frobenius distance of
$A_{\mathcal O}$, not its complete total-variation mixing curve.

\begin{figure}[tbp]
    \centering
    \includegraphics[width=\textwidth]{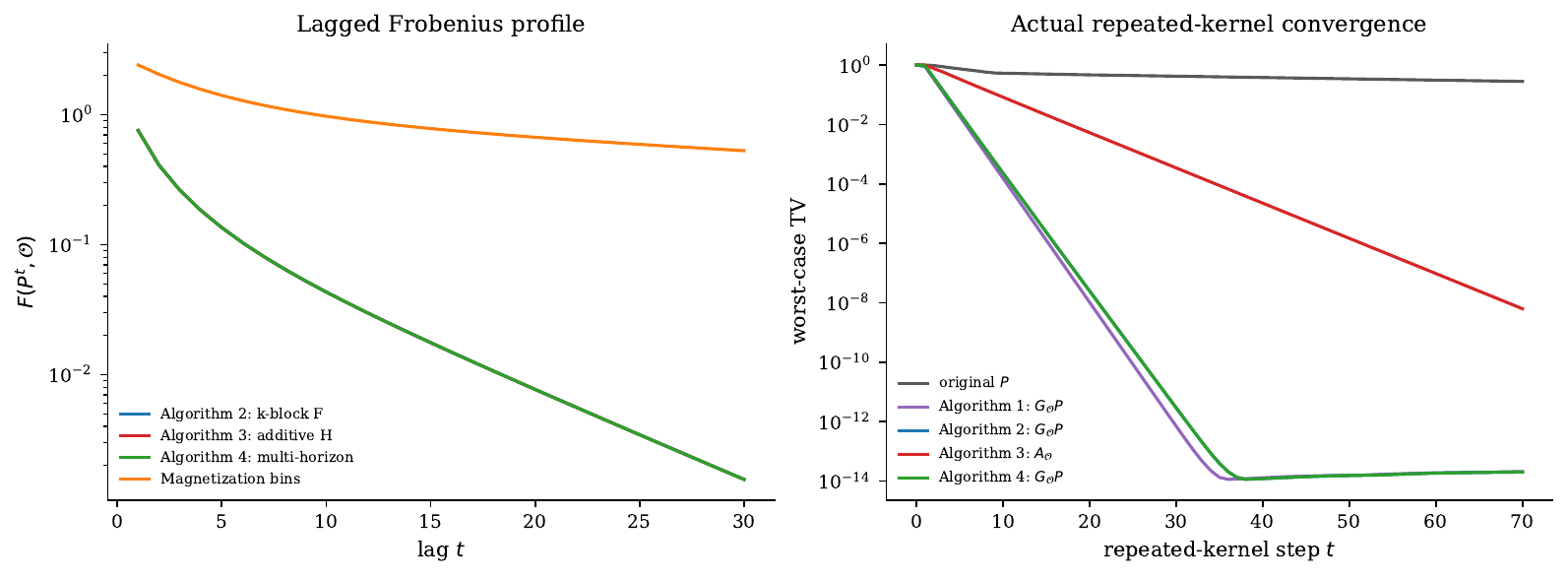}
    \caption{Diagnostics for the exact Curie--Weiss chain at
    $(\beta,h)=(1.5,0)$.  Left: the lagged Frobenius profile
    $F(P^t,\mathcal O)$.  Right: exact worst-case total variation for
    repeated application of each displayed kernel.}
    \label{fig:ising-diagnostics}
\end{figure}

\subsection{Interpretation and limitations of the first two experiments}
\label{subsec:numerical-limitations}

The first two experiments highlight complementary effects: the controlled graph
shows why the additive-mixture embedding must retain the signed spectrum
of $P$ and how one-step and multi-horizon criteria can select different
partitions, while the exact Ising model demonstrates applicability
to a nontrivial Gibbs distribution whose selected partitions need not resemble
magnetization or metastable clusters.  These two studies are diagnostic
rather than exhaustive: they involve only $40$ and $256$ states and unconstrained
objectives may produce unbalanced blocks.  Larger systems will require
sparse eigensolvers, approximate objective evaluation, or symmetry
reduction; moreover, the total-variation plots compare kernel iterations
rather than equal computational budgets, and the Frobenius objectives
alone do not provide a universal total-variation mixing guarantee. 

\subsection{Bayesian variable selection and block balance}
\label{subsec:bvs-balance}

We finally examine whether faster convergence of the Markov kernel
improves statistical estimation.  A state
$m\in\X=\{0,1\}^{12}$ represents a variable-selection model, where
$m_j=1$ means that predictor $j$ is included.  Hence
$|\X|=2^{12}=4096$.

We generate $80$ training observations and $60$ independent test
predictor vectors.  For each training observation
$i\in\llbracket 80\rrbracket$,
generate an independent latent vector
\[
Z_i=(Z_{i1},\ldots,Z_{i6})
\sim \mathcal N_6(0,I_6)
\]
and independent noises
\[
U_{i1},\ldots,U_{i,12}
\stackrel{\mathrm{iid}}{\sim}\mathcal N(0,1).
\]
All latent variables and noises are mutually independent.  For
$r\in\llbracket 6\rrbracket$, define the two raw predictors in pair $r$ by
\[
\begin{aligned}
	X^{\mathrm{raw}}_{i,2r-1}
	&=
	\sqrt{\rho}\,Z_{ir}
	+\sqrt{1-\rho}\,U_{i,2r-1},\\
	X^{\mathrm{raw}}_{i,2r}
	&=
	\sqrt{\rho}\,Z_{ir}
	+\sqrt{1-\rho}\,U_{i,2r},
\end{aligned}
\qquad
\rho=0.995.
\]
Thus the two predictors within each pair have population correlation
$\rho$, whereas different pairs are independent.  The test predictors
are generated independently by the same construction.

Each predictor column is standardized using its training-sample mean
and standard deviation.  Let $X\in\mathbb R^{80\times12}$ denote the
resulting standardized training design matrix; the test predictors
are standardized using the same training means and standard
deviations, and their resulting vectors are denoted by
$x_i^{\mathrm{test}}$,
$i\in\llbracket 60\rrbracket$.

The uncentered training response is generated according to
\[
Y_i^{\mathrm{raw}}
=
1.45Z_{i1}
-1.20Z_{i2}
+1.00Z_{i3}
+0.85\varepsilon_i,
\qquad
\varepsilon_i\stackrel{\mathrm{iid}}{\sim}\mathcal N(0,1),
\]
where the response noises are independent of all latent variables and
proxy noises.  Let $y=(y_1,\ldots,y_{80})^\top$ be the centered
response, where
\[
y_i
=
Y_i^{\mathrm{raw}}
-\frac{1}{80}\sum_{s=1}^{80}Y_s^{\mathrm{raw}}.
\]
Thus the first three predictor pairs carry signal, whereas the
remaining three pairs correspond to latent variables with zero
regression coefficients and are therefore null predictor pairs.

For $m\in\X$, let $|m|=\sum_{j=1}^{12}m_j$, let $X_m$ contain the
columns of $X$ for which $m_j=1$, and let
$\beta_m\in\mathbb R^{|m|}$ be the corresponding coefficient vector.
The model prior has independent coordinates
\[
m_j\sim\operatorname{Bernoulli}(0.16),
\qquad j\in\llbracket 12\rrbracket.
\]
Conditional on $m$, we use the centered Gaussian linear model and a
fixed Zellner $g$-prior with prior multiplier $80$:
\[
\begin{aligned}
	y\mid\beta_m,\sigma^2,m
	&\sim
	\mathcal N\!\left(
	X_m\beta_m,\sigma^2I_{80}
	\right),\\
	\beta_m\mid\sigma^2,m
	&\sim
	\mathcal N\!\left(
	0,\,
	80\sigma^2(X_m^\top X_m)^{-1}
	\right),\\
	p(\sigma^2)
	&\propto
	\frac{1}{\sigma^2}.
\end{aligned}
\]
Centering $X$ and $y$ is equivalent to including a common intercept
with a flat prior and integrating it out.  Integrating out
$(\beta_m,\sigma^2)$ and enumerating all $2^{12}=4096$ models gives
the exact posterior distribution $\pi$ on $\X$.

Let $\widetilde P$ denote the single-toggle Metropolis kernel.  The
baseline kernel is its $0.1$-lazy version, that is,
\[
P=0.1I+0.9\widetilde P.
\]
Equivalently, at each iteration the chain remains at its current model
with probability $0.1$.  With probability $0.9$, one coordinate is
chosen uniformly, its inclusion indicator is toggled, and the proposed
model is accepted with the usual Metropolis probability.  Rejected
proposals also leave the chain at its current model.

We take $k=4$ and $\gamma=0.9$.  Algorithms
\ref{alg:multiway-sweep}, \ref{alg:additive-mixture-sweep} and \ref{alg:multi-horizon-sweep} are applied
with one modification to their rounding step.  In addition to ordinary
weighted $k$-means, we consider a capacity-constrained version
requiring
\begin{equation}
	0.15\leq\pi(O_i)\leq0.35,
	\qquad i\in\llbracket 4\rrbracket.
	\label{eq:bvs-balance-band}
\end{equation}
The spectral embeddings and final rescoring objectives $F$, $H$, and
$F_\gamma^\infty$ remain unchanged. We compare the resulting kernels
\[
G_{\widehat{\mathcal O}_F}P,\qquad
A_{\widehat{\mathcal O}_H},\qquad
G_{\widehat{\mathcal O}_\gamma}P
\]
with the original $0.1$-lazy single-toggle Metropolis kernel $P$, a
separate add/delete/swap Metropolis kernel, and a mass-matched,
$F$-rescored random averaging partition. Here ``mass-matched'' means that the random candidates satisfy the
same admissible block-mass constraints in
\eqref{eq:bvs-balance-band}; their block masses need not coincide
exactly with those of any spectral partition. For this last baseline, we
generate $16$ candidate partitions by drawing independent
$\operatorname{Uniform}(0,1)$ assignment costs for every state--block
pair and minimizing the total assignment cost subject to the same
block-mass constraints in \eqref{eq:bvs-balance-band}.  We then
evaluate $F(\mathcal O)$ for every feasible candidate and retain the
one with the smallest value.  The unconstrained $F$ output is retained
only as a diagnostic.  In this finite experiment, every
$G_{\mathcal O}$ update samples exactly from
$\pi(\,\cdot\mid O_i)$.

\paragraph{Statistical targets.}
Consider one trajectory
$m^{(1)},\ldots,m^{(T)}$.  For predictor $j$, define the exact and
empirical posterior inclusion probabilities by
\[
    p_j=\pi(m_j=1),
    \qquad
    \widehat p_j
    =\frac{1}{T}\sum_{t=1}^T m_j^{(t)}.
\]

For predictor pair $r\in\llbracket 6\rrbracket$ and
$(a,b)\in \{0,1\}^2$, define
\[
    q_r(a,b)
    =
    \pi\!\left(
        (m_{2r-1},m_{2r})=(a,b)
    \right)
\]
and
\[
    \widehat q_r(a,b)
    =
    \frac{1}{T}\sum_{t=1}^T
    \1_{\{
        (m_{2r-1}^{(t)},m_{2r}^{(t)})=(a,b)
    \}}.
\]
Thus $q_r$ and $\widehat q_r$ are probability distributions on the
four inclusion patterns $00$, $01$, $10$, and $11$.

The exact and empirical model-size distributions are
\[
    s_\ell=\pi(|m|=\ell),
    \qquad
    \widehat s_\ell
    =
    \frac{1}{T}\sum_{t=1}^T
    \1_{\{|m^{(t)}|=\ell\}},
    \qquad \ell\in\llbracket 0,12\rrbracket.
\]

Finally, let $\mu_m(x_i^{\mathrm{test}})$ denote the posterior-mean
prediction under model $m$ at test point $x_i^{\mathrm{test}}$.  The
exact and empirical Bayesian-model-averaged predictions are
\[
    \overline\mu_i
    =
    \sum_{m\in\X}
    \pi(m)\mu_m(x_i^{\mathrm{test}})
\]
and
\[
    \widehat\mu_i
    =
    \frac{1}{T}\sum_{t=1}^T
    \mu_{m^{(t)}}(x_i^{\mathrm{test}}),
    \qquad i\in\llbracket 60\rrbracket.
\]

The four statistical errors are
\begin{equation}
    \label{eq:bvs-errors}
    \begin{aligned}
        e_{\mathrm{PIP}}
        &:=
        \left[
            \frac{1}{12}
            \sum_{j=1}^{12}
            (\widehat p_j-p_j)^2
        \right]^{1/2}, \quad
        e_{\mathrm{pair}}
        :=
        \frac{1}{6}\sum_{r=1}^6
        \frac{1}{2}
        \sum_{(a,b)\in \{0,1\}^2}
        \left|
            \widehat q_r(a,b)-q_r(a,b)
        \right|,\\
        e_{\mathrm{size}}
        &:=
        \frac{1}{2}
        \sum_{\ell=0}^{12}
        |\widehat s_\ell-s_\ell|, \quad
        e_{\mathrm{BMA}}
        :=
        \left[
            \frac{1}{60}
            \sum_{i=1}^{60}
            (\widehat\mu_i-\overline\mu_i)^2
        \right]^{1/2}.
    \end{aligned}
\end{equation}
Here $e_{\mathrm{PIP}}$ measures marginal variable-selection error,
$e_{\mathrm{pair}}$ measures error in the posterior uncertainty
between correlated proxies, $e_{\mathrm{size}}$ compares the
model-size distributions, and $e_{\mathrm{BMA}}$ measures error in
the Bayesian-model-averaged predictive mean.

The posterior displays the intended local-mixing obstruction.  For
the three signal-bearing proxy pairs, the posterior probabilities of
the patterns $(00,01,10,11)$ are
\[
\begin{aligned}
    &(0.000,0.249,0.735,0.016),\\
    &(0.000,0.382,0.605,0.013),\\
    &(0.000,0.359,0.627,0.014).
\end{aligned}
\]
Thus one proxy from each pair is almost always present, but the
posterior remains uncertain about which proxy should be selected.
A single-toggle chain can move between $10$ and $01$ only through
one of the disfavoured patterns $00$ or $11$.

\begin{figure}[tbp]
	\centering
	\includegraphics[width=\textwidth]
	{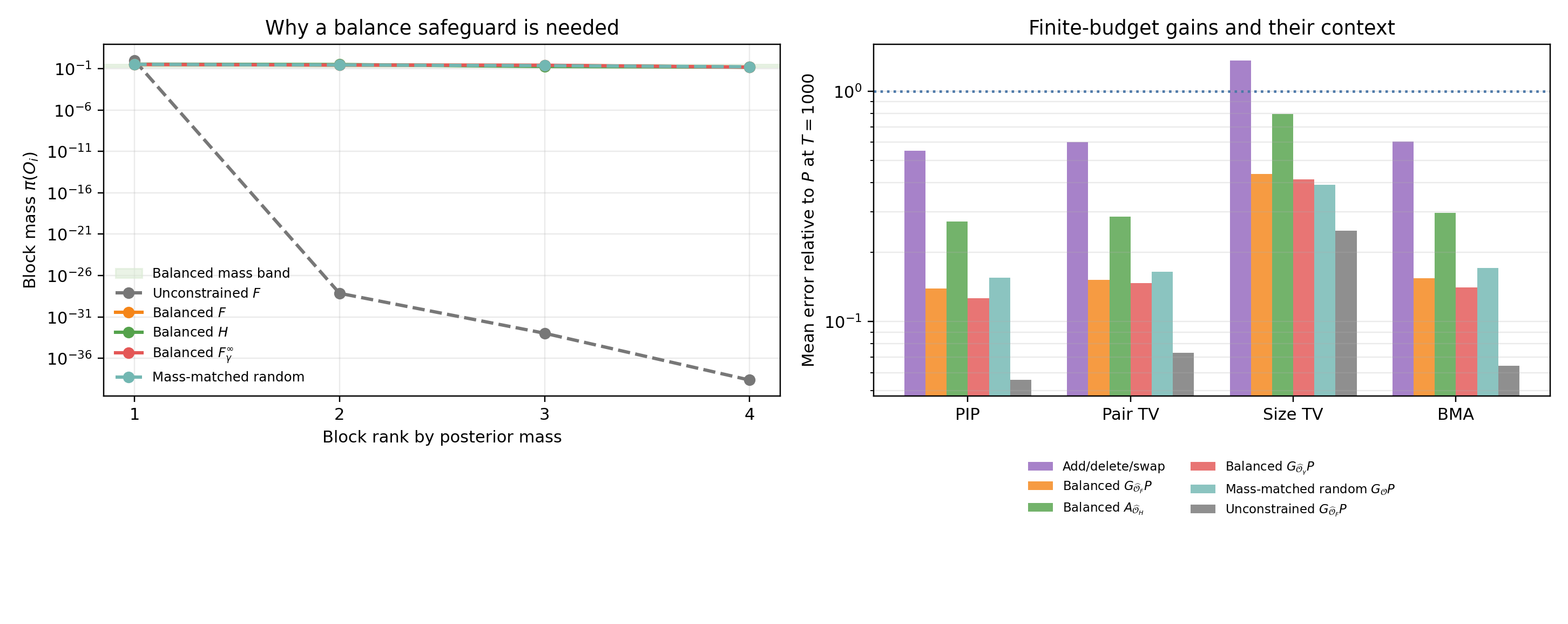}
	\caption{Exact Bayesian variable-selection benchmark with $k=4$.
		Left: block masses in decreasing order; unconstrained rounding
		places almost all posterior mass in one block, whereas the
		constrained outputs obey \eqref{eq:bvs-balance-band}.
		Right: mean values of the four errors in
		\eqref{eq:bvs-errors} at $T=1000$, divided by the corresponding
		mean error for the $0.1$-lazy single-toggle Metropolis kernel $P$.
		The dotted horizontal line at one represents $P$, while the purple
		bars represent the separate add/delete/swap Metropolis benchmark.
		Values below one indicate improvement over $P$. For each kernel, results are averaged over $160$ independent
		length-$1000$ trajectories, all initialized at the
		maximum-a-posteriori model
		$
		m_{\mathrm{MAP}}
		\in
		\operatorname*{argmax}_{m\in\X}\pi(m).
		$}
	\label{fig:bvs-balance}
\end{figure}

Figure~\ref{fig:bvs-balance} explains why the balance constraint is
important.  Unconstrained rounding returns block masses approximately
\[
\left(
    1.00,\,
    7.3\times10^{-29},\,
    1.1\times10^{-33},\,
    2.4\times10^{-39}
\right),
\]
so its Gibbs update is nearly a direct posterior draw.  Its excellent
errors are therefore an oracle-like diagnostic rather than evidence
for a useful partition in practice.

The balanced outputs avoid this collapse and still improve all four
statistical targets.  Although Figure~\ref{fig:bvs-balance} reports
results at $T=1000$, at the longer, unplotted budget $T=5000$, the
balanced multi-horizon kernel reduces $e_{\mathrm{PIP}}$,
$e_{\mathrm{pair}}$, $e_{\mathrm{size}}$, and $e_{\mathrm{BMA}}$ by
factors $10.4$, $8.4$, $2.5$, and $9.5$, respectively, relative to
$P$.  The balanced one-step kernel performs similarly, while the
additive mixture gives smaller but still clear gains.

The mass-matched, $F$-rescored random baseline performs close to the
balanced spectral outputs.  Thus, in this experiment, most of the
improvement over $P$ appears to arise from introducing a nondegenerate
conditional-averaging step with controlled block masses, while the
additional improvement obtained from spectral candidate generation is
comparatively modest.  This conclusion is specific to the present
dataset, balance constraint, and candidate budget.  Moreover, the
balance constraint controls posterior mass rather than simulation
cost.  Exact conditional draws are available here only because all
$4096$ models are enumerated; a scalable application would require
blocks whose conditional posterior distributions can be sampled
without enumerating the full model space.

\section*{Acknowledgements}

Michael Choi acknowledges financial support from the National University
of Singapore through projects A-0000178-02-00 and A-8003574-00-00.

Both authors also acknowledge the use of OpenAI's ChatGPT, powered by
GPT-5.6 Sol, as an AI-assisted research and writing tool during the
development of this manuscript.  It was used to assist with literature
search, the exploration and checking of mathematical arguments,
\LaTeX{} editing, and the development and verification of numerical
code.  All AI-assisted outputs were reviewed by the authors. All scientific content, interpretations, and conclusions remain the sole responsibility of the authors.

\bibliographystyle{plainnat}
\bibliography{spectral_sweep_two_and_k_block_v13}

\end{document}